\documentclass{article} 
\usepackage{iclr2027_conference,times}

\usepackage{amsmath,amsfonts,bm}

\def\eqref#1{equation~\ref{#1}}

\def\1{\bm{1}}

\DeclareMathAlphabet{\mathsfit}{\encodingdefault}{\sfdefault}{m}{sl}
\SetMathAlphabet{\mathsfit}{bold}{\encodingdefault}{\sfdefault}{bx}{n}

\DeclareMathOperator*{\argmin}{arg\,min}

\usepackage{hyperref}
\usepackage{url}
\usepackage{amsthm}
\newtheorem{theorem}{Theorem}
\newtheorem{lemma}[theorem]{Lemma}

\usepackage{graphicx}
\usepackage{multirow}
\usepackage{amsmath,amssymb,amsfonts}
\usepackage{mathrsfs}
\usepackage[title]{appendix}
\usepackage{xcolor}
\usepackage{textcomp}
\usepackage{manyfoot}
\usepackage{booktabs}
\usepackage{algorithm}
\usepackage{algorithmicx}
\usepackage{algpseudocode}
\usepackage{listings}

\usepackage{mathtools}
\usepackage[section]{placeins}  
\usepackage{xr}
\usepackage{wrapfig}
\usepackage{enumitem}
 \usepackage{url}

\title{Higher-order pruning of experts in mixture-of-experts language models}

\author{Alex M. Tseng, Prannay Kaul, Luca Zancato, Wei Xia \& Stefano Soatto\\
AWS Agentic AI\\
\texttt{\{amtseng,prannayk,wxia,soattos\}@amazon.com}; \texttt{zancato@amazon.it}
}

\iclrfinalcopy 
\begin{document}

\maketitle

\begin{abstract}
Mixture-of-Experts (MoE) language models suffer from large parameter counts, which create a significant memory bottleneck. Expert pruning is the most direct approach for reducing this parameter count, yet existing methods make pruning decisions for each expert independently, and assume experts' contributions are purely additive. In reality, expert usage in MoEs is inherently cooperative. We derive HOPE\footnote{Code is available at \url{https://github.com/awslabs/HOPE/}} (Higher-Order Pruning of Experts), a second-order pruning objective which provably minimizes an upper bound on the error resulting from pruning. We show that REAP (a state-of-the-art first-order pruning method) is a special case of HOPE where interaction terms are ignored. Across three frontier MoE models (up to 122B parameters), two distinct calibration sets, and multiple benchmarks (including math, instruction following, coding, and an agentic suite), we demonstrate that HOPE produces better pruning decisions than existing methods, and its advantage is most pronounced at high pruning rates and on challenging agentic workloads. At 50\% pruning, HOPE outperforms all baselines and achieves an average rank of 1.58 out of 5 methods (versus 2.42 for the next-best method, REAP), with gains of up to +6.1\% on agentic coding. Over all conditions, HOPE again achieves the best average rank and surpasses every other method in the majority of head-to-head comparisons. By preserving cooperative expert structure that first-order methods ignore, HOPE enables aggressive compression with minimal degradation, particularly on complex tasks where diverse expert combinations are invoked over long sequences.
\end{abstract}


\section{Introduction}

Mixture-of-expert (MoE) language models have rapidly become the dominant architecture for state-of-the-art LLMs, such as DeepSeek-V3, Qwen3.5, and GLM-4.5 \citep{deepseekv3,qwen3,glm4}. MoEs attempt to scale model capacity while reducing inference costs by multiplexing the feed-forward MLP in each layer into many ``expert" MLPs. Each token is independently routed to a small subset of these specialized experts in each layer, thereby reducing the inference time for each token while still allowing the model to retain high capacity and capability.

Importantly, however, although each token only activates a small subset of experts in each layer, the entire expert pool must reside in GPU memory. That is, the high parameter count of MoEs competes directly with the KV cache for valuable GPU capacity, thus limiting their practical deployment, including: 1) the need for expensive, high-memory hardware; 2) limited concurrent requests to a model; 3) reduced memory budget for context, tool calls, and reasoning.

The vast majority of the parameter count (and thus memory footprint) of MoEs is precisely the experts themselves (116 billion of the 122 billion parameters in Qwen3.5-122B-A10B are expert MLPs, taking $\sim240$ GB). While quantization reduces the size of each parameter, it cannot eliminate redundant experts, and all experts still must reside in memory. \textit{Expert pruning} takes an orthogonal approach: it permanently removes entire experts, freeing memory regardless of precision, and it composes with quantization \citep{reap2025,chen2025collaborative}. Crucially, pruning is one-shot and requires no retraining. Pruning Qwen3.5-122B by 50\% would immediately reduce its footprint to 125 GB, freeing 115 GB for additional context and tool calls, longer reasoning traces, and deeper searches. This substantially expands the model's effective capabilities on the same hardware.

Current expert-pruning methods have already achieved promising results, but they are limited in their independent treatment of experts. These methods are \textit{first order}, meaning they assign each expert a scalar importance score in isolation, ignoring interactions between experts. In reality, MoE inference is inherently combinatorial: each token activates a group of $K$ experts whose outputs are jointly aggregated, and certain pairs of experts are co-selected far more often than independent chance would predict (Figure~\ref{fig:S2pmi}). First-order methods have a major failure mode arising from this cooperative structure. If an expert $A$ is pruned, cooperative circuits it participates in are disrupted. Partner experts that are frequently co-selected with $A$ may no longer combine coherently, potentially degrading performance beyond what removing $A$ alone would suggest. At aggressive pruning rates (40--50\%), many such cooperative pairs are broken simultaneously, compounding damage in ways that first-order importance scores cannot predict or prevent. Higher pruning rates are also the most impactful for deployment: halving the expert pool immediately frees enough memory for longer contexts and extended reasoning, or to accommodate substantially more concurrent requests.

To this end, we derive HOPE (\textbf{H}igher \textbf{O}rder \textbf{P}runing of \textbf{E}xperts), a second-order pruning objective which considers pairwise interactions between experts to inform pruning decisions. During calibration, we record expert usage, including interaction terms that reflect how pairs of experts complement each other, resulting in a square interaction matrix per layer which is solved as a quadratic program with minimal computational overhead. We show that REAP, a state-of-the-art first-order method, is a special case of HOPE where interaction terms are ignored. Across extensive experiments (3 MoE architectures up to 122B parameters, 2 calibration sets, 6 pruning rates, and multiple benchmarks), HOPE achieved the best performance most frequently and beat every state-of-the-art baseline in the majority of head-to-head comparisons. As predicted by our theory, HOPE's margin over competing methods grows with pruning rate (particularly at 40--50\%), with its advantage most pronounced in complex agentic tasks which rely on diverse expert combinations. Our analysis also confirms that HOPE's advantage stems from selectively retaining experts with stronger cooperative structure---exactly the signal that first-order methods miss.

\section{Related Work}

\subsection{Mixture-of-experts}

In a Transformer-based MoE, the MLP in each layer is replaced by $E$ expert MLPs and a learned router. For a layer with $E$ experts, the output is $h(x) = \sum_k g_{k}(x) f_{k}(x)$, where $f_{k}(x) \in \mathbb{R}^{d}$ is the output of expert $k$ and $g_{k}(x) \in [0,1)$ is its gate weight. The layer selects the top-$K$ experts with the highest router logits (forming the selected set $T(x)$), and these logits are softmax-normalized into gate weights: $g_{k}(x) > 0$ iff $k \in T(x)$, with $\sum_{k=1}^{E} g_{k}(x) = 1$. This allows high total parameter count while activating far fewer parameters per token. Standard training of MoEs includes a load-balancing loss \citep{shazeer2017} to encourage uniform expert usage in each layer, although empirically, utilization remains highly non-uniform across experts (and this enables expert pruning).

\subsection{Expert pruning}

Expert pruning permanently removes experts to reduce the memory footprint of MoEs \textit{without} retraining. Existing methods follow a common recipe: pass a calibration set $\mathcal{D}$ through the model, collect per-expert usage statistics in each layer, and assign each expert $k$ a scalar importance score $S_{k}$. A fixed fraction of experts per layer is then pruned according to the lowest $S_{k}$. These methods differ in the statistics which are collected and how $S_{k}$ is defined.

\textbf{Frequency}: prune the least-activated experts. Rank experts by $S_{k}^{\text{freq}} = \vert\mathcal{X}_{k}\vert$, where $\mathcal{X}_{k}$ is the set of tokens which activate expert $k$. This method is simple, but ignores expert outputs entirely.

\textbf{EAN} (Expert Activation Norm) \citep{jaiswal2025}: prune by the sum of expert output norms over active tokens. Rank experts by $S_{k}^{\text{EAN}} = \sum_{x\in \mathcal{X}_{k}}\Vert f_{k}(x)\Vert_{2}$.

\textbf{REAP} (Router-weighted Expert Activation Pruning) \citep{reap2025}: prune by gate-weighted activation norm, averaged over active tokens. Rank experts by $S_{k}^{\text{REAP}} = \frac{1}{\vert \mathcal{X}_{k}\vert}\sum_{x\in\mathcal{X}_{k}}g_{k}(x)\Vert f_{k}(x)\Vert_{2}$. This is generally considered the state-of-the-art pruning method for generative tasks, and is backed by theoretical justification.

\textbf{MAN} (Mean Activation Norm) \citep{liu2025man}: prune by activation norm, averaged over active tokens. Rank experts by $S_{k}^{\text{MAN}} = \frac{1}{\vert \mathcal{X}_{k}\vert}\sum_{x\in\mathcal{X}_{k}}\Vert f_{k}(x)\Vert_{2}$. This is very similar to REAP, but without the gate weighting. \citet{liu2025man} also placed existing first-order pruning methods into a common framework parameterized by routing frequency, gate weighting, and activation strength.

All of these pruning methods are first-order: they score each expert independently using a scalar importance metric $S_{k}$, ignoring expert interactions.

Note that alternatives to pruning---namely expert \textit{merging}---have also been proposed to reduce MoE parameter count, but \citet{reap2025} showed both theoretically and empirically that expert merging leads to irreducible error and underperforms expert pruning. Intuitively, pruning (as opposed to merging) preserves the router's ability to independently modulate surviving experts.

\subsection{Expert interaction and cooperation}

Non-pruning work has shown that expert co-selection patterns are stable and exploitable properties of MoEs. Flame-MoE \citep{flamemoe2025} showed that expert co-selection patterns emerge early in training and stabilize throughout. \citet{symphonysmoe2025} augmented the MoE router with an expert co-selection graph that encourages repeated activation of cooperative expert pairs, improving robustness. These results establish that pairs of experts \textit{cooperate}---an exploitable property---but no prior work has leveraged this cooperative structure to inform pruning decisions.

\section{HOPE: Higher-Order Pruning of Experts}

In this section, we briefly outline the theoretical justification behind the HOPE objective. The full derivation can be found in Appendix~\ref{app:hope_deriv}.

\subsection{Setup and notation}

Consider an MoE layer with $E$ experts. For input token $x\in\mathbb{R}^{d}$, the layer output is $h(x) = \sum_{k=1}^{E}g_{k}(x)f_{k}(x)$, where $f_{k}(x)\in\mathbb{R}^{d}$ is the output of expert $k$, and $g_{k}(x)\in[0,1)$ is the gate weight. The router maps $x$ to a logit per expert and selects the $K$ highest-logit experts to form the selected set $T(x)\subseteq\{1,\ldots,E\}$, with $\vert T(x)\vert = K$. The logits of experts in $T(x)$ are softmax normalized to obtain gate weights $g_{k}(x)$, with $g_{k}(x) > 0 \Leftrightarrow k \in T(x)$ and $\sum_{k\in T(x)}g_{k}(x) = 1$.

Now suppose we delete a \textit{prune-set} $P\subset\{1,\ldots,E\}$ of experts from the layer. For a token $x$, the pruned experts in $P\cap T(x)$ (i.e. those which \textit{would} have been selected if not removed) are replaced by the next-highest-logit experts $R(x)$, which are now in top-$K$ ($\vert R(x)\vert = \vert P\cap T(x)\vert$). The new selected set is $(T(x) \setminus P) \cup R(x)$, with updated gate weights $g_{k}'(x)$, with $g_{k}'(x) > 0 \Leftrightarrow k \in (T(x) \setminus P) \cup R(x)$. The \textit{pruned} layer's output is $h'(x) = \sum_{k=1}^{E}g_{k}'(x)f_{k}(x)$. Our goal is to find the prune-set $P$ (given a prescribed budget $\vert P\vert$) that minimizes the \textit{pruning error} $\Vert h(x) - h'(x)\Vert$.

\subsection{Error decomposition}

Following \citet{reap2025}, we decompose the pruning error of a MoE layer into two components:
$$h(x) - h'(x) = \bigl[\sum_{j \in P \cap T(x)} g_j(x) f_j(x) - \sum_{i \in R(x)} g_i'(x) f_i(x)\bigr] + \sum_{k \in T(x) \setminus P} (g_k(x) - g_k'(x)) f_k(x)$$
The first term is the \textit{substitution error}, which arises from swapping the pruned experts in $P\cap T(x)$ for their replacements $R(x)$. The second term is the \textit{renormalization error}, which arises from rescaling the gates of the retained experts. 

As with \citet{reap2025}, we note that the renormalization error is typically small (we also provide a proof of this in Appendix~\ref{app:hope_deriv}), and so we focus on minimizing the substitution error.

\subsection{The second-order objective}

Our goal is to find the prune-set $P$ which minimizes the substitution error:

\begin{equation}
    P^{*} = \argmin_{P} \bigl\|\sum_{j \in P \cap T(x)} g_{j}(x) f_{j}(x) - \sum_{i \in R(x)} g_{i}'(x) f_{i}(x)\bigr\|_{2}^{2}
    \label{eq:min-sub-error}
\end{equation}

Importantly, we minimize the \textit{squared} substitution error, which reveals interaction terms. The substitution error itself is intractable to optimize directly due to its reliance on $R(x)$ and $g_{i}'(x)$, which depend on the unknown prune-set $P$. We therefore define a tractable upper bound that depends only on pruned experts' contributions (which are observable without committing to $P$):

\begin{equation}
    Z = \bigl[\sum_{j \in P\cap T(x)} g_{j}(x) \cdot \Vert f_j(x)\Vert_{2}\bigr]^{2}
    \label{eq:z-def}
\end{equation}

Intuitively, $Z$ measures the total joint contribution of the prune-set. Expanding the square reveals pairwise interaction terms between pruned experts. Note that $Z$ depends on input tokens $x$ and the prune-set $P$. This definition of $Z$ admits the following upper bound on the substitution error:

\begin{theorem}
For any pruning set $P$, the squared substitution error is bounded by:
\[
\bigl\Vert\sum_{j \in P \cap T(x)} g_j(x) f_j(x) - \sum_{i \in R(x)} g_i'(x) f_i(x)\bigr\Vert_2^2 \leq (1 + \rho)^2 \cdot Z \quad\text{where}\quad \rho = \frac{\max\limits_{k} \Vert f_k(x)\Vert_{2}}{\min\limits_{k} \Vert f_k(x)\Vert_{2}}
\]
\label{thm:z_bound}
\end{theorem}

$Z$ is a tractable surrogate whose minimization provably reduces the (squared) substitution error. The factor $(1+\rho)^2$ (independent of $P$) arises from bounding the contribution of replacement experts $R(x)$, whose identity depends on $P$ and is unknown at optimization time. This factor is conservative: it assumes a worst-case replacement whose output norm equals the layer maximum. In practice, the gap between $Z$ and the true error is far smaller than the bound suggests.
  
\subsection{Minimizing $Z$}

Given the upper bound in Theorem~\ref{thm:z_bound}, our goal is to find the prune-set $P$ which minimizes $Z$. More precisely, we minimize the expectation of $Z$ over the calibration set $\mathcal{D}$: $P^{*} = \argmin_{P}\mathop{\mathbb{E}}_{x\in \mathcal{D}}[Z]$. To minimize $\mathop{\mathbb{E}}_{x\in \mathcal{D}}[Z]$, we substitute the expanded form of $Z$ (Equation~\ref{eq:z-def}) into the expectation, and replace the expectation with the empirical average. This minimization can equivalently be formulated as a binary quadratic program (QP) over the matrix $F\in\mathbb{R}^{E\times E}$, defined as follows:

\begin{equation}
    F_{i,j} = \frac{1}{N} \sum_{x\in\mathcal{X}_{i,j}} g_{i}(x)g_{j}(x)\Vert f_{i}(x)\Vert_{2}\Vert f_{j}(x)\Vert_{2}\quad\text{where}\quad \mathcal{X}_{i,j} = \{x: i,j \in T(x)\}
    \label{eq:f-def}
\end{equation}

The $F$-matrix is the core calibration result used by HOPE. $F$ encodes both individual expert contributions (on the diagonal) and pairwise expert co-contributions (on the off-diagonal). High $F_{i,j}$ implies that pruning experts $i,j$ together incurs extra error from their reinforcing contributions. To formulate the QP, we introduce binary decision variables $p_{k} \in \{0,1\}$ where $p_{k} = 1$ iff expert $k \in P$.

\begin{theorem}
Given a target pruning budget $\vert P\vert$ (the number of experts to remove per layer), the prune-set $P^{*}$ which minimizes $\mathop{\mathbb{E}}_{x\in \mathcal{D}}[Z]$ is also the solution to the following quadratic program:
\[
P^{*} = \argmin_{p}\; p^{\top}Fp \quad \text{s.t.} \quad p \in \{0,1\}^{E}, \sum\limits_{k=1}^{E}p_{k} = \vert P\vert
\]
\label{thm:qp}
\end{theorem}

\textbf{Normalization}: This formulation of $F$ (Equation~\ref{eq:f-def}) minimizes $\mathbb{E}[Z]$ exactly (Theorem~\ref{thm:qp}). In practice, we replace the \textit{unconditional} average over the total number of tokens, $N$, with a \textit{conditional} average, $\vert\mathcal{X}_{i,j}\vert$. This follows the analogous design choice in REAP, which prevents rarely co-selected expert pairs (which nonetheless contribute strongly) from being undervalued. 

\subsection{Connection to REAP}

HOPE's $F$ matrix encodes individual expert importance scores on the diagonal, and this diagonal effectively recovers the scores used by REAP to independently rank experts for pruning. Specifically, $F_{k,k} = \mathbb{E}[(g_{k}(x)\Vert f_{k}(x)\Vert)^{2}]$, whereas the squared REAP score is $(S_{k}^{\text{REAP}})^{2} = \mathbb{E}[g_{k}(x)\Vert f_{k}(x)\Vert]^{2}$. These expressions differ only by the variance of expert contributions over activating tokens. Indeed, HOPE with zeroed off-diagonals empirically recovers the same prune-set as REAP (Figure~\ref{fig:S4fdiag_vs_reap}). In contrast with REAP, HOPE uses the full $F$, where $F_{i,j}$ ($i\neq j$) encodes the interaction between these experts when they are both activated. Thus, any improvement that HOPE achieves over REAP is directly attributable to these off-diagonal interaction terms which encode second-order effects.

\subsection{Practical implementation}

As with other expert-pruning methods, HOPE requires a single forward pass over the calibration set. For each token, we record the active experts, their gate values, and their outputs. We also collect interaction terms, which are the product of the (gate-weighted) expert outputs when two experts are co-activated. For each layer, this gives us an ($E\times E$) $F$-matrix. We solve each layer's QP (using a standard QP solver) via a continuous relaxation with the conditions in Theorem~\ref{thm:qp}, rounding the continuous solution to the top $\vert P\vert$ entries. The relaxation is empirically tight: continuous solutions concentrate near 0 and 1, and the continuous--binary gap is negligible (Figure~\ref{fig:S5qp_relaxation}). In terms of computational cost, HOPE's calibration step is dominated by forward passes (shared with all methods), and solving the QP adds negligible additional cost (1--2 seconds per layer, for typical values of $E$).

\section{Experiments}

\subsection{Experimental setup}
\label{sec:exp-setup}

\begin{itemize}[leftmargin=1em]
    \item \textbf{Models}: Qwen3.5-122B-A10B (256 experts, top-8 routing, 48 MoE layers), Qwen3.5-35B-A3B (256 experts, top-8 routing, 40 MoE layers), GLM-4.5-Air (128 experts, top-8 routing with 1 shared expert, 45 MoE layers); note that shared experts are never pruned
    \item \textbf{Baselines}: REAP \citep{reap2025}, EAN \citep{jaiswal2025}, MAN \citep{liu2025man}, Frequency
    \item \textbf{Calibration data}: Evol-CodeAlpaca-v1 (coding-focused) \citep{wizardcoder2023}, SWE-Bench verified trajectories (agentic-focused) \citep{swebench2024}
    \item \textbf{Pruning rates}: 10\%, 20\%, 25\%, 30\%, 40\%, 50\% (prune rate of 10\% means removing 10\% of experts in each layer); for SWE-Bench Pro, we only test 10\%, 25\%, and 50\%
    \item \textbf{Benchmarks}: Tulu3 Dev suite (GSM8K, MATH, IFEval, MMLU, BBH, TruthfulQA, PopQA, LiveCodeBench) \citep{tulu3_2024,livecodebench2024}, SWE-Bench Pro \citep{deng2025swebenchpro}
\end{itemize}

We tested HOPE on several models ranging from 35 billion to well over 100 billion parameters, across multiple architectural families. We focused our benchmarks on agentic tasks, as well as standard coding, math, instruction following, general knowledge, etc. This gives 54 total configurations (3 models $\times$ 2 calibration sets $\times$ 6 pruning rates for Tulu, or 3 pruning rates for SWE-Bench Pro). 

\subsection{HOPE performance comparison}

\begin{figure}[h]
\centering
\includegraphics[width=\columnwidth]{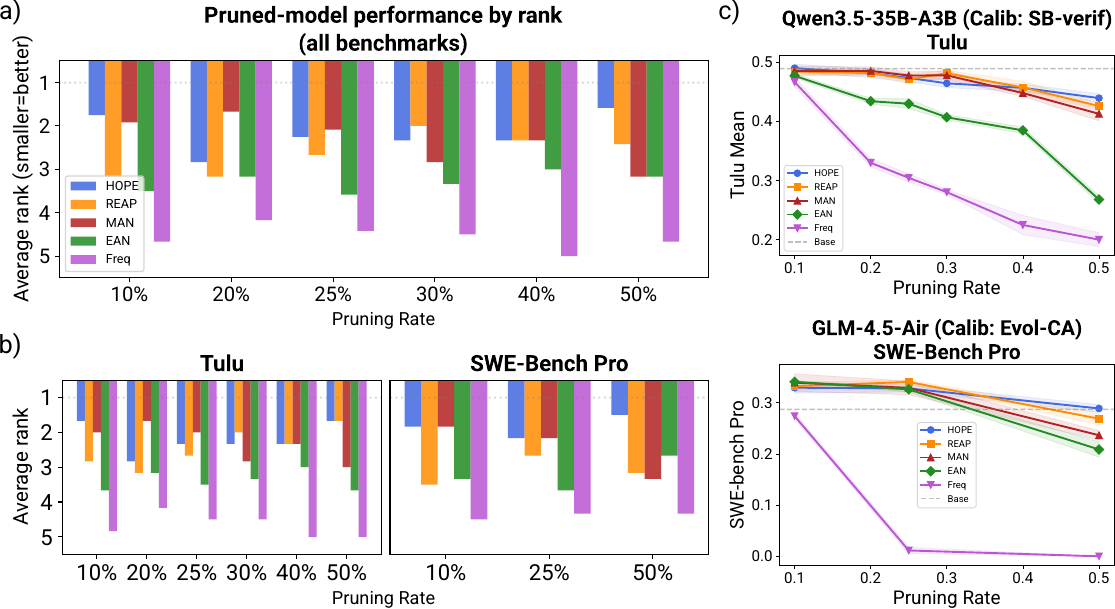}
\caption{\small HOPE achieves the best overall performance across pruning conditions, and dominates at aggressive pruning rates. \textbf{a)} Average rank of HOPE compared to other baseline pruning methods, averaged over models, calibration sets, and benchmarks (54 conditions total). \textbf{b)} Average ranks separated by benchmark: Tulu3 and SWE-Bench Pro. \textbf{c)} Performance vs pruning rate for representative conditions; dashed gray line indicates the unpruned base model’s performance. Note: pruned models occasionally match or slightly outperform the unpruned baseline; this is consistent with observations that calibration can act as implicit specialization, removing low-relevance experts whose contributions add noise to the target tasks \citep{dong2025easyep}.}
\label{fig:1perf}
\end{figure}

We tested HOPE over a dense set of configurations and find that HOPE is the overall best-performing method, particularly at higher pruning rates (Figure~\ref{fig:1perf}). Over all conditions, HOPE had the best average rank of 2.07, finishing as the top-1 method or in the top-2 more often than any other method (top-1 in 39\% of conditions, top-2 in 67\% of conditions). Furthermore, in a head-to-head comparison, HOPE beat every other individual baseline in the majority of conditions, with an average head-to-head win rate of 73\% (Figure~\ref{fig:S1perf_full}, Tables~\ref{tab:s1tulu_qwen3.5-35b}--\ref{tab:s1swebenchpro_glm4.5}).

At more aggressive pruning rates (40--50\%), HOPE achieved an average rank of 1.58, and finished as the top-1 method in the majority of conditions. This advantage was particularly marked in the agentic benchmark, where HOPE beat REAP in all experiments at a high pruning rate (40--50\%), with a mean gap of +2.8\% and up to +6.1\%. This result matches our theory: HOPE's advantage over first-order methods should grow with pruning rate. As more experts are removed, more interactions are broken, and the off-diagonal terms of $F$ become more important. Higher pruning rates are also the most deployment-relevant settings, where memory savings are most crucial.

HOPE consistently improved over REAP, and this directly validates that the off-diagonal interaction terms provide signal which improves performance beyond REAP (over all conditions and pruning rates, REAP had a top-1 rate of 15\% and a top-2 rate of 48\%). MAN was also a strong competitor overall at lower pruning rates, but it lost its edge at more aggressive pruning rates. Furthermore, HOPE is robust: it finished last in only 1/54 conditions (20\%-pruned Qwen3.5-122B-A10B), where the gap to the best method was only -1.7\%. In contrast, REAP had 5 last-place finishes and Freq had 43. HOPE's worst deficit on any single condition was -4.5\%, and its deficits are smaller and far less frequent, while its gains reach up to +6.1\% (Figure~\ref{fig:S1perf_full}, Tables~\ref{tab:s1tulu_qwen3.5-35b}--\ref{tab:s1swebenchpro_glm4.5}).

Interestingly, HOPE's advantage was most pronounced on agentic coding (SWE-bench Pro), code generation (LiveCodeBench), and compositional reasoning (BBH), where gains over the next-best method reached +2.0\%, +19.6\%, and +11.7\% respectively. On simpler recall tasks (e.g. PopQA), the advantage of second-order pruning was smaller. This pattern is consistent with our theoretical motivation: tasks requiring diverse expert combinations over many tokens may benefit most from preserving cooperative structure.

\subsection{Structure of HOPE's $F$-matrix}

\begin{figure}[t]
\centering
\includegraphics[width=0.8\columnwidth]{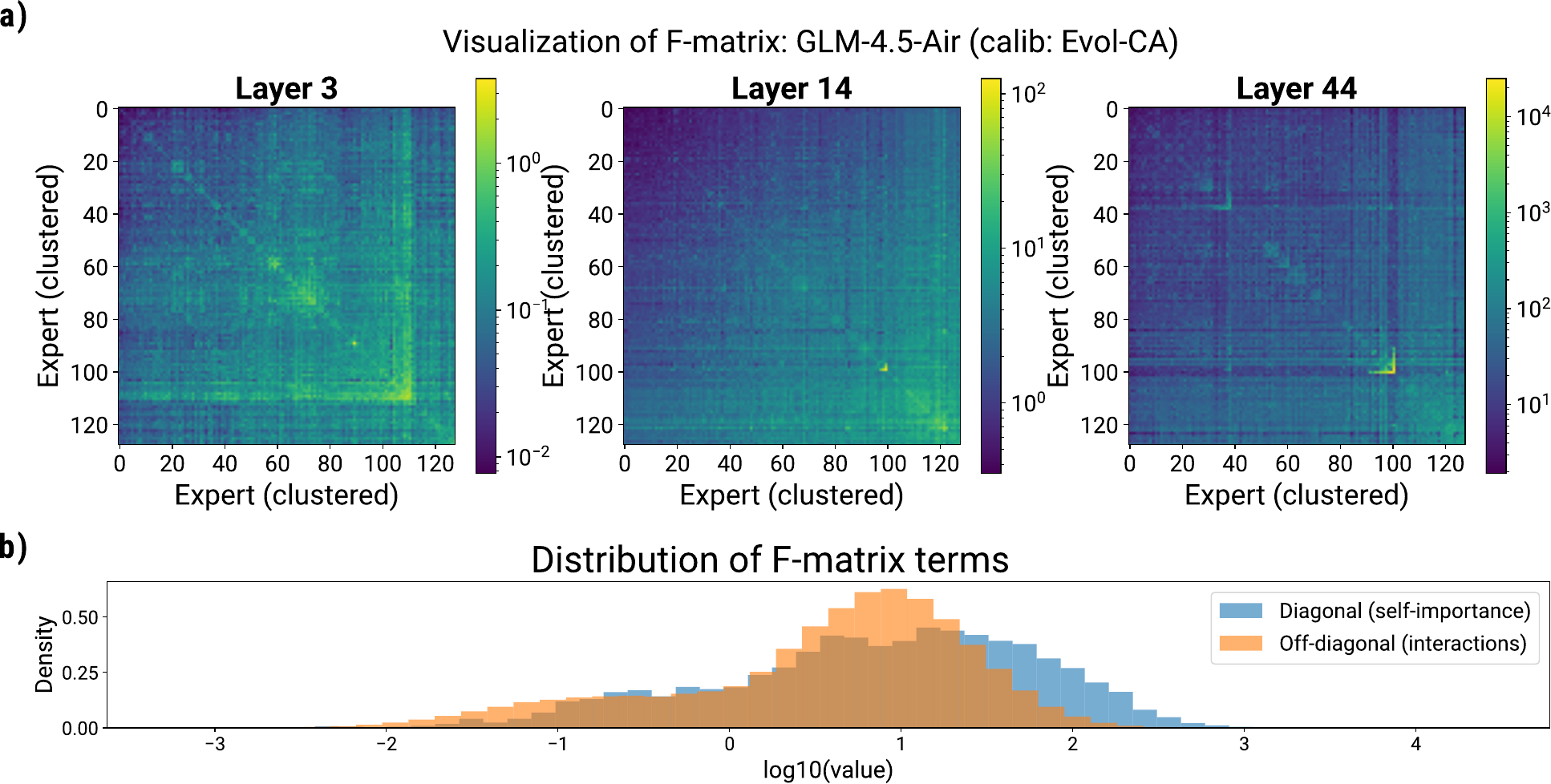}
\caption{\small The interaction matrix $F$ recorded by HOPE exhibits expert cooperation. \textbf{a)} $F$-matrix for GLM-4.5-Air at three representative layers (early, middle, late). Experts have been hierarchically clustered per layer for visualization. Visible block structure indicates groups of experts with high pairwise interaction scores. \textbf{b)} Distribution of diagonal vs off-diagonal terms in $F$ across all layers (GLM-4.5-Air). Off-diagonal (interactive) terms are substantial in magnitude (compared to the diagonal), showing that pairwise interactions are a non-negligible component of the pruning objective.}
\label{fig:2fmatrix}
\end{figure}

Having shown that HOPE outperforms first-order pruning methods---particularly at high pruning rates---we now examine the structure of the $F$-matrix that drives this advantage. Importantly, the off-diagonal entries of $F$ are not noise: they exhibit noticeable substructure  (Figure~\ref{fig:2fmatrix}a), and their average magnitude is 33\% of the diagonal (Figure~\ref{fig:2fmatrix}b). Clusters of experts with high $F_{i,j}$ tend to be co-selected and contribute jointly. These interaction terms vary by layer but are consistently present throughout the network. HOPE accounts for this structure that first-order methods ignore entirely, leading to substantially different prune-sets.

\begin{figure}[t]
\centering
\includegraphics[width=0.8\columnwidth]{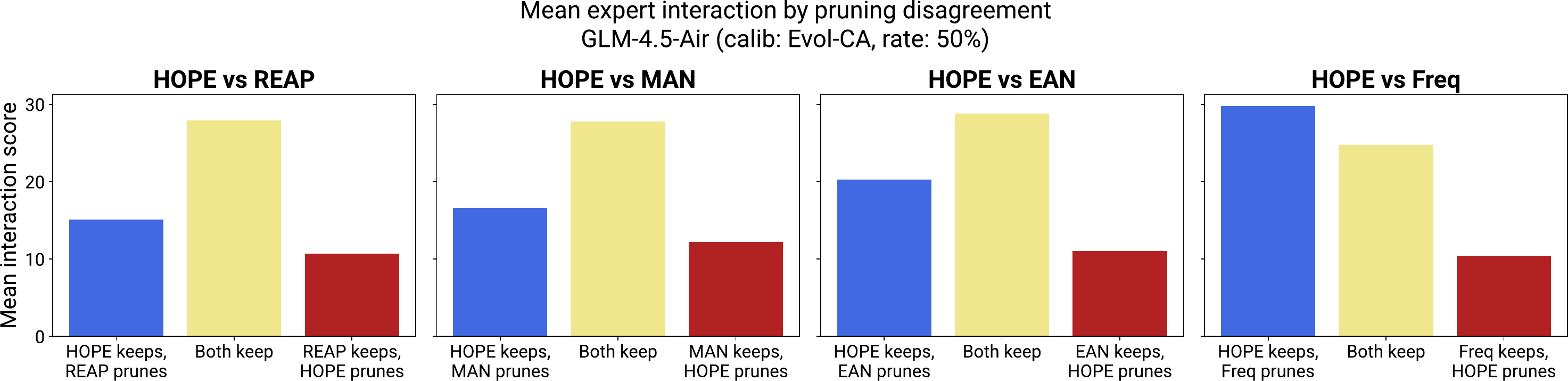}
\caption{\small HOPE preferentially retains experts with stronger cooperative structure. For each baseline, we classify experts by disagreement (kept only by HOPE, kept only by the baseline, or kept by both). Experts kept by HOPE tend to have higher cooperative score (measured as the mean $F$ score with HOPE-retained experts).}
\label{fig:3exp_interaction_breakdown}
\end{figure}

Indeed, when HOPE and a first-order method disagree, HOPE preferentially retains experts with stronger interactions with other surviving experts (measured as the mean $F$ score) (Figure~\ref{fig:3exp_interaction_breakdown}). Mechanistically, the QP assigns a higher cost to pruning experts which participate in cooperative clusters, causing them to be retained even if their individual importance (diagonal of $F$) is lower. In contrast, experts retained by a first-order method (e.g. REAP) but pruned by HOPE have high individual importance but lower interaction scores with other retained experts. As such, first-order methods tend to be misled by experts' self-importance and retain them regardless of cooperative structure.

\begin{figure}[t]
\centering
\includegraphics[width=0.7\columnwidth]{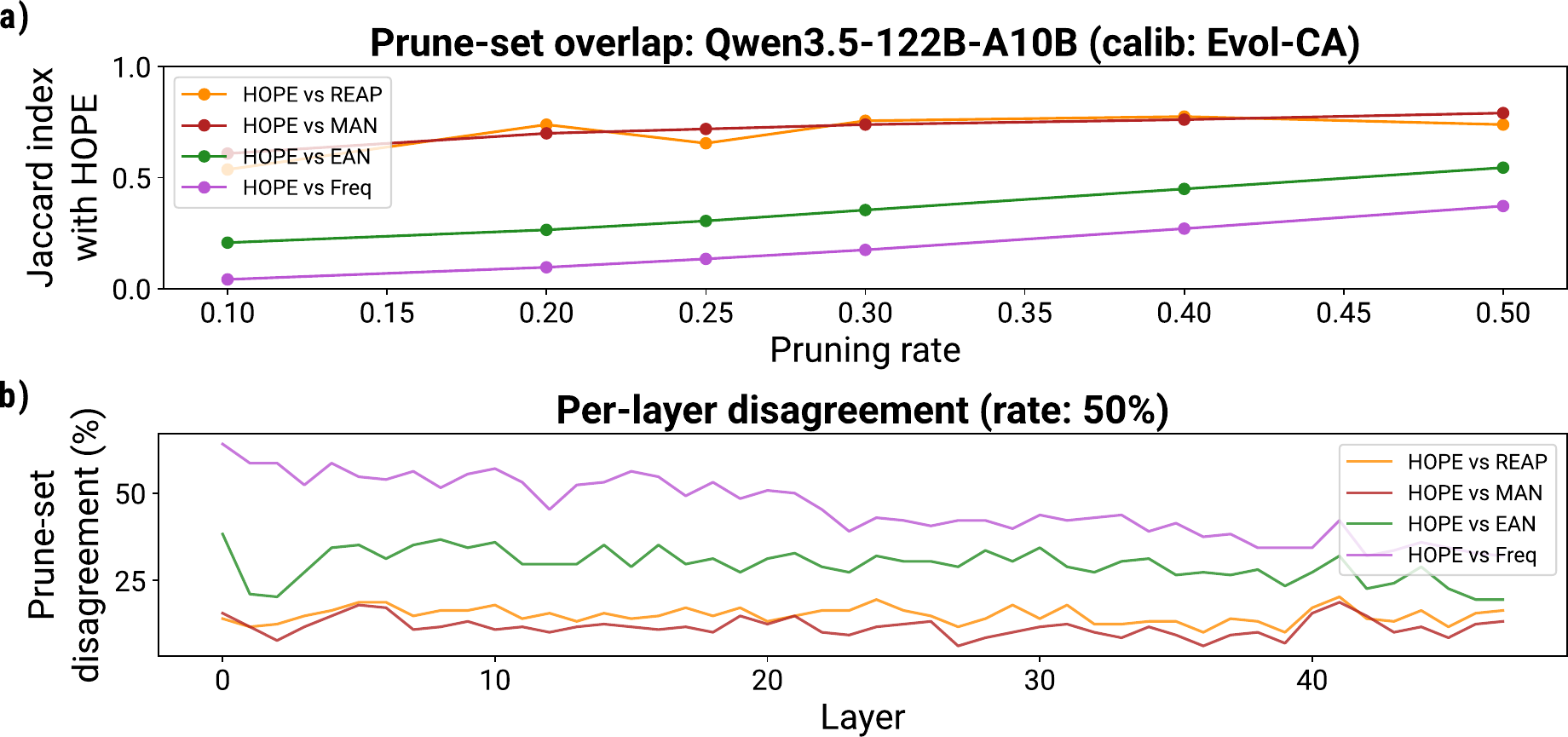}
\caption{\small HOPE makes distinct pruning decisions from other baseline methods. \textbf{a)} Overlap between HOPE’s pruning set versus each baseline, as a function of pruning rate (Qwen3.5-122B-A10B). We measure overlap as the Jaccard index, due to the prune-set size varying over different prune rates. \textbf{b)} For the same model and calibration set, the percentage of disagreeing pruned experts (per layer) between HOPE versus other methods at a pruning rate of 50\%. Disagreement is distributed across the full network.}
\label{fig:4blacklist_overlap}
\end{figure}

HOPE's interaction-aware pruning decisions lead to substantially distinct prune-sets compared to other methods, but HOPE's prune-set is least distant from REAP and MAN (all three methods rely on expert outputs $f(x)$ averaged over expert-activating tokens) (Figure~\ref{fig:4blacklist_overlap}a). As expected, prune-set overlap between methods increases with pruning rate: as more experts are pruned, all methods are forced to remove the agreed-upon unimportant experts, and thereby converge on a larger shared set. At the same time, however, the remaining disagreements at higher pruning rates become more critical: the impact of each marginal pruning decision affects a larger fraction of the surviving experts, who have fewer cooperative partners. Thus, HOPE's interaction-aware decisions yield the largest gains in this high-prune-rate regime. 

Disagreement between HOPE and other methods is also distributed across all layers (Figure~\ref{fig:4blacklist_overlap}b), consistent with the observation that interactions are ubiquitous throughout the network (Figure~\ref{fig:2fmatrix}a, Figure~\ref{fig:S2pmi}). This indicates that second-order pruning can be valuable across the network.

\subsection{Calibration and trial robustness}

\begin{figure}[t]
\centering
\includegraphics[width=0.5\columnwidth]{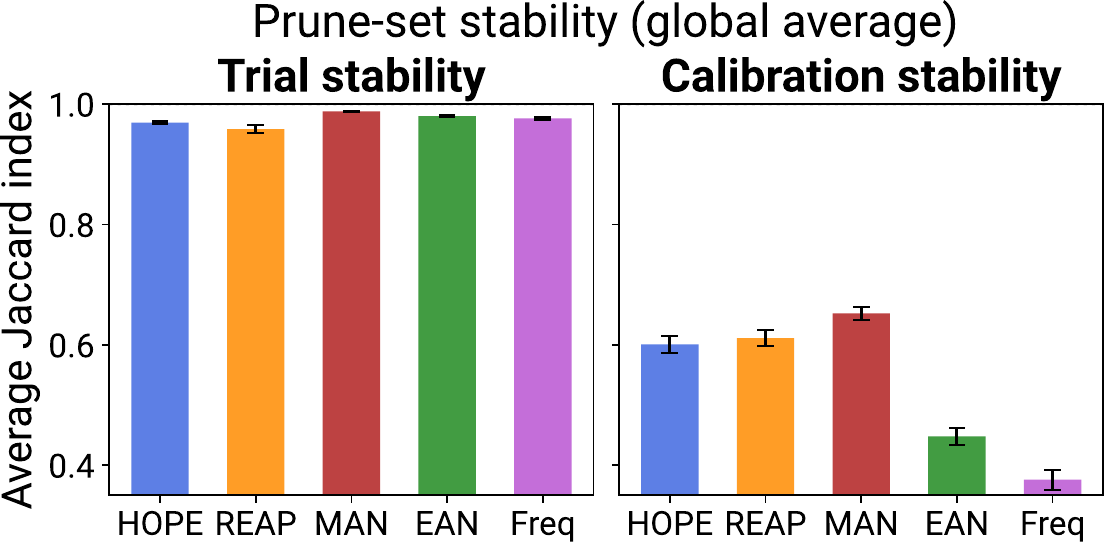}
\caption{\small Prune-set stability across calibration sets and trials. \textbf{a)} Jaccard index across independent trials of the same method (averaged over models, pruning rates, and calibration sets). \textbf{b)} Jaccard index between prune-sets produced from different calibration sets: Evol-CodeAlpaca and SWE-Bench verified. Jaccard index is averaged over models, pruning rates, and trials (higher is more stability).}
\label{fig:5sensitivity}
\end{figure}

Expert cooperation---as captured by HOPE's $F$-matrix---is a stable property of the model, not an artifact of calibration noise. Over 3 random trials, prune-set stability is high ($> 0.95$ Jaccard similarity on average) for all methods including HOPE (Figure~\ref{fig:5sensitivity}a, Figure~\ref{fig:S3blacklist_overlap_grid}). This confirms that the interaction structure HOPE exploits is reliably recoverable from calibration data. As expected, stability across different calibration sets is somewhat lower: different calibration domains lead to slightly different pruning decisions (Figure~\ref{fig:5sensitivity}b). This effect is shared across all methods, and HOPE is not more sensitive than REAP (HOPE's average calibration Jaccard index is 0.60, versus REAP's 0.61).

\subsection{Compatibility with downstream training}


\begin{figure}[h]
\centering
\includegraphics[width=0.4\columnwidth]{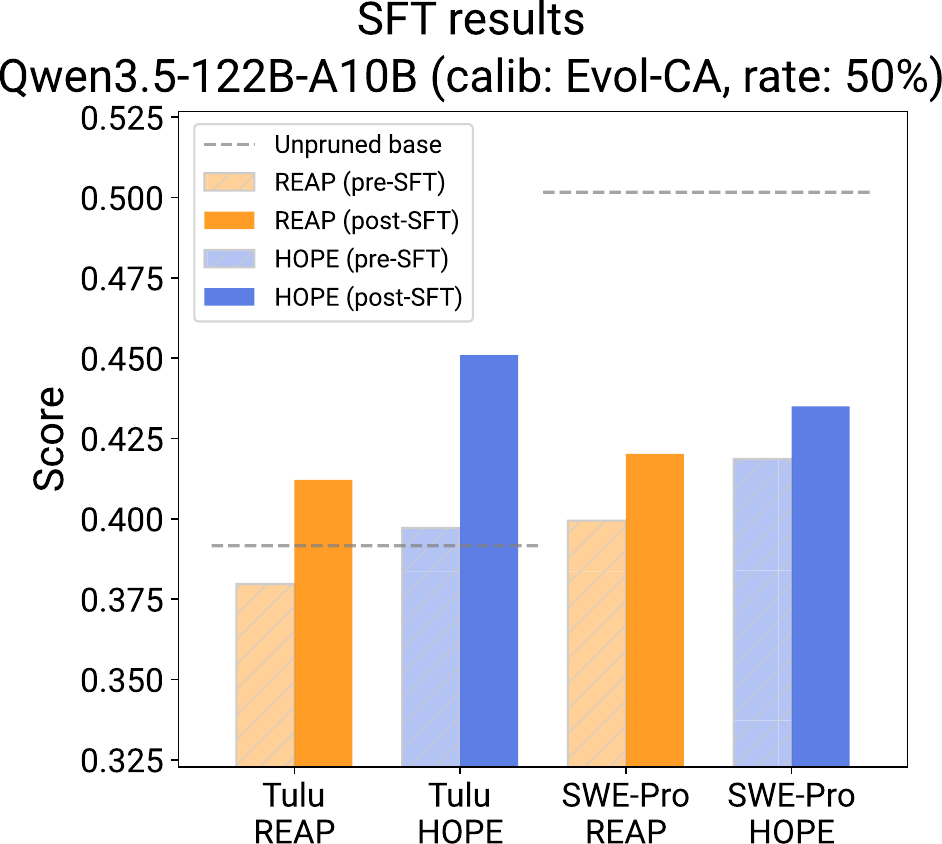}
\caption{\small SFT results for pruned Qwen3.5-122B-A10B models. Pre-SFT and post-SFT use matched trials for a fair comparison. The gray dashed line indicates the performance of the unpruned model.}
\label{fig:6sft}
\end{figure}

Finally, we show that HOPE-pruned models serve as strong starting points for downstream training. Crucially, downstream training is not required to realize HOPE's advantage; we include this experiment to show that HOPE's structural advantage persists even after training. On REAP- or HOPE-pruned checkpoints of Qwen3.5-122B-A10B, we performed SFT on LiveCodeBench traces (collected from the unpruned model) for 1 billion tokens of training. HOPE outperformed REAP both before and after SFT (Figure~\ref{fig:6sft}): that is, fine-tuning narrows the gap but does not close it. This suggests that the expert structure preserved by second-order pruning provides lasting benefits that fine-tuning alone cannot replicate.

\section{Conclusion}

We presented HOPE, the first theoretically derived expert-pruning objective that directly accounts for pairwise expert interactions, formulated as a quadratic program over an interaction matrix. We showed that HOPE generalizes REAP and outperforms it (along with other methods), especially at aggressive pruning rates where interaction effects dominate. Our analysis confirms that HOPE's pruning decisions are largely influenced by off-diagonal interaction terms, which direct the selective retention of experts with high cooperative scores. Our results validate our theoretical prediction: considering these interactive terms---a signal which all first-order methods miss---in expert-pruning decisions yields better prune-sets. Importantly, HOPE's advantage over other methods is greatest in the high-pruning-rate regime, which is most relevant to real-world deployment.

\textbf{Limitations:} Calibration with HOPE does require additional time and memory due to its quadratic nature, although the additional cost is small. HOPE's calibration step takes about 6--7\% longer than REAP on the same-size calibration set (Table~\ref{tab:S2timing}). The additional memory cost of HOPE to store the $F$-matrix is also negligible for the number of experts we typically have (128--256). HOPE does require marginally more data than REAP due to its estimating inter-expert (quadratic) in addition to per-expert statistics. The difference, however, is minor: to converge to 99\% agreement of the same prune-set, REAP requires 8k prompts whereas HOPE requires 10k prompts (Qwen3.5-122B-A10B, calib: Evol-CA, rate: 50\%) (Figure~\ref{fig:S6f_convergence}). Finally, like most other methods, HOPE prunes experts in each layer independently, and does not account for cross-layer interactions.

\textbf{Future work:} We offer a theoretical framework (and preliminary results) for extending HOPE to cross-layer interactions and cross-layer pruning (Appendix~\ref{app:chope}). Cross-layer HOPE enables non-uniform pruning budgets in each layer, potentially allocating more aggressive pruning to layers with weaker cooperative structure. In practice, however, this framework may introduce non-uniform budget-allocation instability (e.g. over-pruning layers with weak cooperative structure), which may require some additional constraints to address. Additionally, as MoEs continue to scale to larger models---where aggressive pruning is even more important for deployment---HOPE's advantage could grow even further. Finally, we note that HOPE operates only up to second-order (pairwise) interactions, which already demonstrably provide signal beyond first-order methods. Extending HOPE to even higher-order interactions could be valuable, although they are combinatorially difficult to quantify.

\clearpage

\subsection*{AI use statement}

In this work, we used generative AI to:
\begin{itemize}[leftmargin=1em]
    \item Search existing literature for relevant directions and baselines
    \item Check human-written code for bugs
    \item Write the first draft of code specifically for creating plots/figures
    \item Suggest rewordings for specific sentences in the manuscript
    \item Insert citations into the manuscript and download *.bib citations
    \item Reformat CSV files of collected data/results into \LaTeX-formatted tables (via an AI-generated script)
\end{itemize}

We did \textit{not} use AI to write any code which created any main results, other than to identify bugs (which were subsequently verified and fixed by a human). AI was used to write the first draft of code for creating plots and figures, and this code was then verified and revised by a human.

We did \textit{not} use AI to generate research ideas, derive theoretical proofs, design experiments, or interpret results.

We did \textit{not} use AI to write any part of the text for the manuscript, other than as a ``copy-editor'' to search for typos, insert citations, and suggest rewordings for specific sentences or sections when prompted. AI was used to automate the formatting of \LaTeX\, tables from saved CSV files.

We take responsibility for the final content of this work.

\subsection*{Ethics statement}

This research focuses on algorithmic compression of MoE language models. All datasets (calibration or evaluation) are publicly available and do not contain personal information. All models used are also publicly available and open source. Our work's contribution is methodological: we derived and demonstrated a method for reducing the memory footprint of MoEs. Although we hope our work will improve accessibility of generative AI, we also acknowledge that such algorithms could potentially lower barriers to deploying models/agents for harmful uses, as well. Additionally, we pruned pre-trained models, and any biases already present in those models may persist after pruning. In general, we believe our work does not introduce any new ethical considerations beyond those already present in the efficient deployment of LLMs.

\subsection*{Reproducibility statement}

Code for HOPE calibration, QP solving, and all analyses will be released upon publication. The theoretical derivation of HOPE is provided in full in Appendix~\ref{app:hope_deriv}. Section~\ref{sec:exp-setup} details the experimental setup, including models, calibration datasets, pruning rates, and evaluation benchmarks. Appendix~\ref{app:methods} provides supplementary methods describing all analysis procedures in sufficient detail for reproduction, including hyperparameters, solver settings, and evaluation configurations. All experiments use publicly available models and datasets.

\clearpage

\bibliography{main}
\bibliographystyle{iclr2027_conference}

\clearpage

\appendix

\renewcommand{\thefigure}{S\arabic{figure}}
\renewcommand{\thetable}{S\arabic{table}}
\setcounter{figure}{0}
\setcounter{table}{0}

\section{Full derivation of HOPE}
\label{app:hope_deriv}
\subsection{Setup}

For an arbitrary MoE layer with $E$ experts, the output of that layer on input-token representation $x$ is:
\[
h(x) = \sum\limits_{k=1}^{E}g_{k}(x) f_{k}(x)
\]
where $g_{k}(x) \in [0,1)$ is the softmax-normalized gate weight and $f_{k}(x) \in \mathbb{R}^{d}$ is the expert output. We have $g_{k}(x) > 0$ iff $k \in T(x)$ (the top-$K$ selected set), and $\sum\limits_{k=1}^{E} g_{k}(x) = 1$.

Now suppose we prune a set of experts $P$. For any token $x$, the pruned experts $P \cap T(x)$ are replaced by the next-highest-logit experts $R(x)$, with $\vert R(x)\vert = \vert P \cap T(x)\vert$. After pruning, the new gate weights are $g_{k}'(x) > 0$ iff $k \in (T(x) \setminus P) \cup R(x)$, and the new layer output is:
\[
h'(x) = \sum\limits_{k=1}^{E} g_{k}'(x) f_{k}(x)
\]

\subsection{Error decomposition}

The pruning error decomposes as:
\begin{align*}
h(x) - h'(x) &= \underbrace{\bigl[\sum_{j \in P \cap T(x)} g_{j}(x) f_{j}(x) - \sum_{i \in R(x)} g_{i}'(x) f_{i}(x)\bigr]}_{\text{substitution error}} +
\underbrace{\sum_{k \in T(x) \setminus P} (g_{k}(x) - g_{k}'(x)) f_{k}(x)}_{\text{renormalization error}}
\end{align*}

The first term is the \textit{substitution error} and the second is the \textit{renormalization error}. The substitution error is the direct effect of swapping experts, and is the dominant source of error. The renormalization error is the indirect effect of pruning on untouched experts. Because gate values $g_{k}(x)$ must sum to 1, pruning and replacing experts causes the gate values of non-pruned experts in the top-$K$ set ($T(x) \setminus P$) to be rescaled. This particular decomposition was described in \citep{reap2025}. From this point forward, the derivations will diverge substantially.

\subsection{Gate-mass lemma}

An intermediate result which will be helpful is to show that replacement experts in $R(x)$ have a lower total gate mass. More formally:

\begin{lemma}[Replacement experts have lower gate mass]
\label{lem:gate-mass}
Under top-$K$ routing with softmax normalization, suppose experts $P \cap T(x)$ are pruned and replaced with experts $R(x)$. Then $\sum\limits_{i\in R(x)}g_{i}'(x) \leq \sum\limits_{j\in P\cap T(x)}g_{j}(x)$.
\end{lemma}

\textbf{Proof}: Define $l_{k}(x)$ be the router logit for expert $k$. Note this is a function of the input token and the router, and is unchanged by pruning.

Recall, after pruning, the new selected set becomes $T'(x) = (T(x) \setminus P)\cup R(x)$. We define the following quantities:
\begin{itemize}
    \item $S_{P} = \sum\limits_{j\in P\cap T(x)}\exp(l_{j}(x))$ (sum of exponentiated logits for pruned experts)
    \item $S_{T} = \sum\limits_{k\in T(x)\setminus P}\exp(l_{k}(x))$ (sum of exponentiated logits for retained experts)
    \item $S_{R} = \sum\limits_{i\in R(x)}\exp(l_{i}(x))$ (sum of exponentiated logits for replacement experts)
\end{itemize}

Before pruning, the gate weights are a softmax over $T(x)$. Then the total gate mass of the pruned experts (before pruning) is $\sum\limits_{j\in P\cap T(x)}g_{j}(x) = \frac{S_{P}}{S_{T} + S_{P}}$.

After pruning, the gate weights are a softmax over $T'(x)$. Then the total gate mass of the replacement experts (after pruning) is $\sum\limits_{i\in R(x)}g_{i}'(x) = \frac{S_{R}}{S_{T} + S_{R}}$.

Crucially, the logit $l_{j}(x)$ for any pruned expert $j\in P\cap T(x)$ must be \textit{at least} the logit $l_{i}(x)$ for any replacement expert $i\in R(x)$. This is because every pruned expert $j\in P\cap T(x)$ had a top-$K$ logit, and every replacement expert did not have a top-$K$ logit. Since exponentiation is a monotonically increasing function, we have that $S_{P} \geq S_{R}$.

To complete the proof for Lemma~\ref{lem:gate-mass}, we note that we simply need to show $\frac{S_{P}}{S_{T} + S_{P}} \geq \frac{S_{R}}{S_{T} + S_{R}}$. Since all quantities are strictly positive, we can cross-multiply and see that this is true precisely when $S_{P} \geq S_{R}$ (which we showed just above).

\subsection{Renormalization error is small}

Recall, the pruning error (which we wish to minimize) decomposes into the substitution error and renormalization error. Next, let us show that the renormalization error is generally small.

For simplicity, we assume the new gate values are renormalized directly from the old ones (without retaking the softmax):
\[
g_{k}'(x) = \frac{g_{k}(x)}{1 - \sum\limits_{j \in P\cap T(x)} g_{j}(x) + \sum\limits_{i \in R(x)} g_{i}'(x)}
\]

Substituting into the renormalization error:
\[
\sum\limits_{k \in T(x) \setminus P} (g_{k}(x) - g_{k}'(x)) f_{k}(x) = \frac{\sum\limits_{i \in R(x)} g_{i}'(x) - \sum\limits_{j \in P\cap T(x)} g_{j}(x)}{\sum\limits_{i \in R(x)} g_{i}'(x) - \sum\limits_{j \in P\cap T(x)} g_{j}(x) + 1} \sum\limits_{k \in T(x) \setminus P} g_{k}(x) f_{k}(x)
\]

Note that the scalar factor in front is small when $\sum_{j \in P\cap T(x)} g_{j}(x) \approx \sum_{i \in R(x)} g_{i}'(x)$. Generally, this is expected because we typically prune experts with the lower gate probabilities among the top-$K$, and replace them with other experts whose gate probabilities barely missed the cut-off for top-$K$. We therefore focus on minimizing the substitution error.

\subsection{Expanding the squared substitution error}

We focus on minimizing the \textit{squared} substitution error, as this reveals interaction terms between pruned experts. We decompose the quadratic form into three components:

\begin{align*}
&\biggl\Vert\sum_{j \in P \cap T(x)} g_{j}(x) f_{j}(x) - \sum_{i \in R(x)} g_{i}'(x) f_{i}(x)\biggr\Vert_{2}^{2} \\
&= \underbrace{\biggl\Vert\sum_{j \in P \cap T(x)} g_{j}(x) f_{j}(x)\biggr\Vert_{2}^{2}}_{A} - \underbrace{2\biggl\langle \sum_{j \in P \cap T(x)} g_{j}(x) f_{j}(x),\; \sum_{i \in R(x)} g_{i}'(x) f_{i}(x)\biggr\rangle}_{B} + \underbrace{\biggl\Vert\sum_{i \in R(x)} g_{i}'(x) f_{i}(x)\biggr\Vert_{2}^{2}}_{C}
\end{align*}

For ease of notation, we split the substitution error into three components: $A + B + C$.

Intuitively, $A$ is the contribution of the pruned experts. Note $A$ is quadratic, and the cross-terms of $A$ captures how much pruned experts reinforce each other. If two pruned experts contribute cooperatively, their joint removal incurs a penalty which is worse than their individual removals.

$C$ is the contribution of the replacement experts, and it measures the contribution of these experts to the output. $C$ also contributes to error because the replacements are not perfectly aligned with what was pruned.

Finally, $B$ is the alignment between pruned experts and retained experts. If replacement experts are well-aligned with the pruned experts, $B$ is large in magnitude and contributes a large negative to the substitution error.

\subsection{Upper-bounding the substitution error via $Z$}

By the Cauchy--Schwarz inequality, we have that $B \leq \vert B\vert\leq 2\sqrt{A}\sqrt{C}$. Substituting, we have the following inequality:
\[
\bigl\Vert\sum_{j \in P \cap T(x)} g_{j}(x) f_{j}(x) - \sum_{i \in R(x)} g_{i}'(x) f_{i}(x)\bigr\Vert_{2}^{2} = A + B + C \leq A + 2\sqrt{A}\sqrt{C} + C = \bigl(\sqrt{A} + \sqrt{C}\bigr)^{2}
\]

We now upper-bound both $\sqrt{A}$ and $\sqrt{C}$ in terms of a common quantity. Define:
\[
Z = \bigl(\sum_{j \in P\cap T(x)} g_{j}(x) \|f_{j}(x)\|_{2}\bigr)^{2}
\]

\textbf{Bounding} $\sqrt{A}$:

By the triangle inequality:
\[
\sqrt{A} = \bigl\Vert\sum_{j \in P \cap T(x)} g_{j}(x) f_{j}(x)\bigr\Vert_{2} \leq \sum_{j \in P \cap T(x)} g_{j}(x) \Vert f_{j}(x)\Vert_{2} = \sqrt{Z}
\]

\textbf{Bounding} $\sqrt{C}$:

Again by the triangle inequality:
\[
\sqrt{C} = \bigl\Vert\sum_{i \in R(x)} g_{i}'(x) f_{i}(x)\bigr\Vert_{2} \leq \sum_{i \in R(x)} g_{i}'(x) \Vert f_{i}(x)\Vert_{2}
\]

We have $\sum\limits_{i \in R(x)} g_{i}'(x) \Vert f_{i}(x)\Vert_{2} \leq \max_{k} \Vert f_{k}(x)\Vert_{2}\sum\limits_{i \in R(x)} g_{i}'(x)$, by replacing the replacement experts' output norms with the maximum norm over all experts in the layer.

By Lemma~\ref{lem:gate-mass}, $\max_{k} \Vert f_{k}(x)\Vert_{2}\sum\limits_{i \in R(x)} g_{i}'(x) \leq \max_{k}\Vert f_{k}(x)\Vert_{2}\sum\limits_{j\in P\cap T(x)}g_{j}(x)$.

Furthermore, for each pruned expert $j$, we also have that $g_{j}(x) \leq g_{j}(x)\frac{\Vert f_{j}(x)\Vert_{2}}{\min_{k}\Vert f_{k}(x)\Vert_{2}}$. Substituting in for $g_{j}(x)$, we have:
\begin{align*}
\max_{k}\Vert f_{k}(x)\Vert_{2}\sum\limits_{j\in P\cap T(x)}g_{j}(x) &\leq \max_{k}\Vert f_{k}(x)\Vert_{2}\sum\limits_{j\in P\cap T(x)}g_{j}(x)\frac{\Vert f_{j}(x)\Vert_{2}}{\min_{k}\Vert f_{k}(x)\Vert_{2}}\\
&= \frac{\max_{k} \Vert f_{k}(x)\Vert_{2}}{\min_{k} \Vert f_{k}(x)\Vert_{2}}\sqrt{Z}
\end{align*}

This substitution provides a loose (but strict) upper bound, without relying on knowledge of the replacement experts, which makes it tractable to optimize. Thus, $\sqrt{C} \leq \rho\sqrt{Z}$, where $\rho = \frac{\max_{k} \Vert f_{k}(x)\Vert_{2}}{\min_{k} \Vert f_{k}(x)\Vert_{2}}$.

\paragraph{Together}:
\[
\bigl\Vert\sum_{j \in P \cap T(x)} g_{j}(x) f_{j}(x) - \sum_{i \in R(x)} g_{i}'(x) f_{i}(x)\bigr\Vert_{2}^{2} \leq (\sqrt{A} + \sqrt{C})^{2} \leq (\sqrt{Z} + \rho\sqrt{Z})^{2} = (1 + \rho)^{2} \cdot Z
\]

This completes the proof of Theorem~\ref{thm:z_bound}.

\subsection{Minimizing $Z$ as a quadratic program}

Our goal is now to minimize $\mathop{\mathbb{E}}\limits_{x \sim \mathcal{D}}[Z]$ over the calibration set $\mathcal{D}$. We introduce binary decision variables $p_{k} \in \{0,1\}$, where $p_{k} = 1$ iff expert $k \in P$, and indicator variables $\mathbf{1}[k \in T(x)]$ which denote if expert $k$ is in the top-$K$.

Substituting:
\begin{align*}
Z &= \bigl(\sum\limits_{j=1}^{E} p_{j} \cdot \mathbf{1}[j \in T(x)] \cdot g_{j}(x) \Vert f_{j}(x)\Vert_{2}\bigr)^{2} \\
&= \sum\limits_{j=1}^{E} \sum\limits_{k=1}^{E} p_{j} p_{k} \cdot \mathbf{1}[j \in T(x)] \cdot \mathbf{1}[k \in T(x)] \cdot g_{j}(x) g_{k}(x) \Vert f_{j}(x)\Vert_{2} \Vert f_{k}(x)\Vert_{2}.
\end{align*}

Taking the empirical expectation over the calibration set (with $N$ tokens):
\begin{align*}
\mathop{\mathbb{E}}\limits_{x}[Z] &= \frac{1}{N} \sum\limits_{x} \sum\limits_{j=1}^{E} \sum\limits_{k=1}^{E} p_{j} p_{k} \cdot \mathbf{1}[j \in T(x)] \cdot \mathbf{1}[k \in T(x)] \cdot g_{j}(x) g_{k}(x) \Vert f_{j}(x)\Vert_{2} \Vert f_{k}(x)\Vert_{2} \\
&= \sum\limits_{j=1}^{E} \sum\limits_{k=1}^{E} p_{j} p_{k} \underbrace{\frac{1}{N} \sum\limits_{x\in \mathcal{X}_{j,k}} g_{j}(x) g_{k}(x) \Vert f_{j}(x)\Vert_{2} \Vert f_{k}(x)\Vert_{2}}_{F_{j,k}} \\
&= p^{\top} F p
\end{align*}
where we define:
\[
F_{i,j} = \frac{1}{N}\sum\limits_{x \in \mathcal{X}_{i,j}} g_{i}(x) g_{j}(x) \Vert f_{i}(x)\Vert_{2} \Vert f_{j}(x)\Vert_{2}, \quad \mathcal{X}_{i,j} = \{x : i, j \in T(x)\}
\]
and $p$ is a binary vector of size $E$.

Note that in practice, we use \textit{conditional} normalization in $F$ instead of unconditional normalization. This equates to the following definition of $F$:
\[
F_{i,j} = \frac{1}{\vert\mathcal{X}_{i,j}\vert}\sum\limits_{x \in \mathcal{X}_{i,j}} g_{i}(x) g_{j}(x) \Vert f_{i}(x)\Vert_{2} \Vert f_{j}(x)\Vert_{2}
\]

This is the empirical choice which we found helps decouple interaction strength from co-selection frequency. This helps ensure that specialist expert pairs which are rarely co-selected---but contribute strongly---are not undervalued, consistent with the analogous choice made by REAP.

Note that if an expert pair is never activated ($\mathcal{X}_{i,j} = \emptyset$), then $F_{i,j} = 0$.

Therefore, finding the optimal prune-set of size $\vert P\vert$ reduces to the binary quadratic program:
\[
P^{*} = \argmin_{p \in \{0,1\}^{E},\; \sum\limits_{k} p_{k} = \vert P\vert}\; p^{\top} F p
\]

This completes the proof of Theorem~\ref{thm:qp}.

\clearpage

\section{Supplementary Figures and Tables}

\begin{figure}[h]
\centering
\includegraphics[width=\columnwidth]{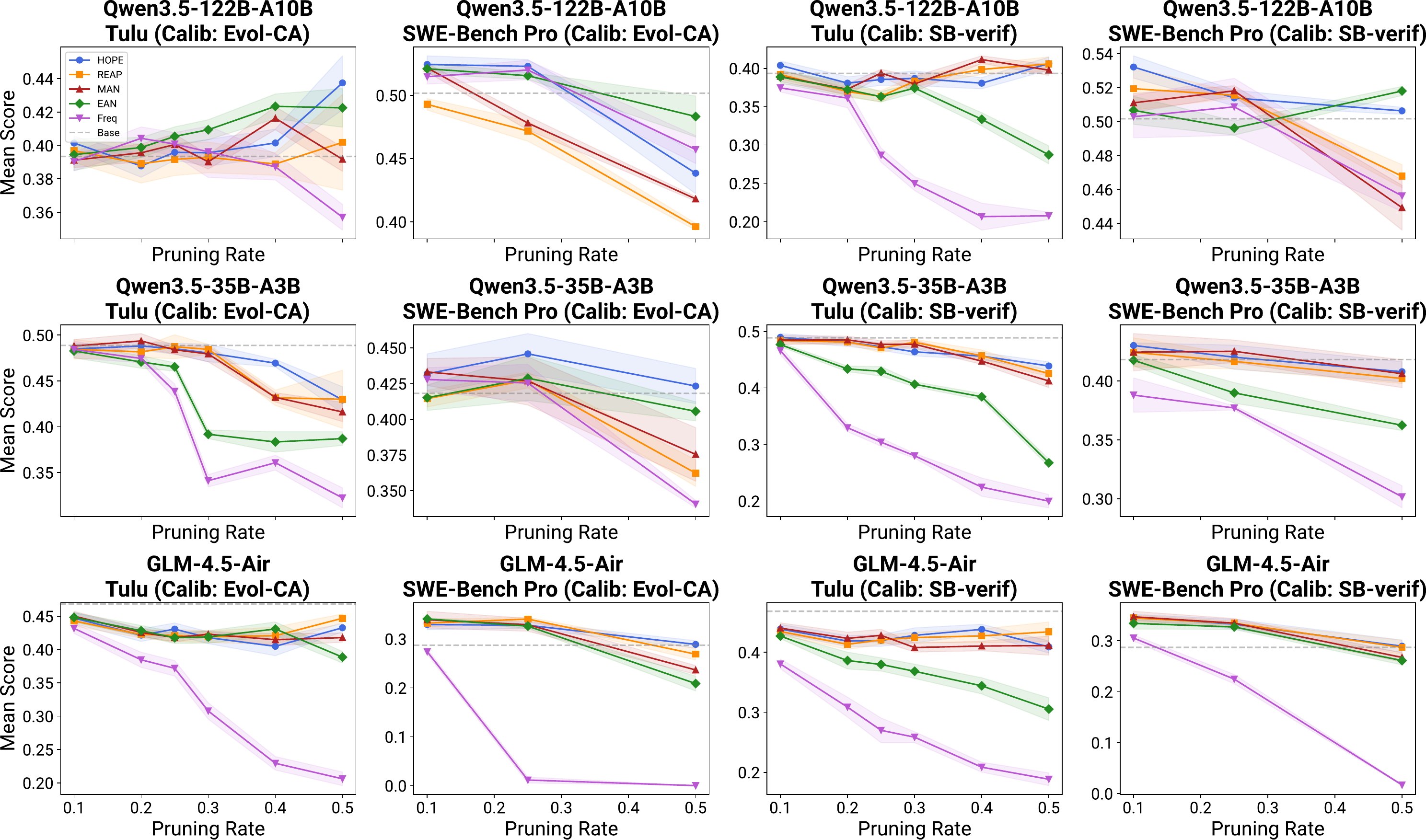}
\caption{\small Full performance results across all models and calibration sets. Each row corresponds to a model architecture (Qwen3.5-122B, Qwen3.5-35B, GLM-4.5-Air). Columns alternate between mean Tulu score and SWE-Bench Pro score for each calibration set (Evol-CodeAlpaca and SWE-Bench verified). The dashed gray line indicates the performance of the unpruned base model. HOPE (blue) is competitive or best at all pruning rates, with the cleanest separation at aggressive pruning rates.}
\label{fig:S1perf_full}
\end{figure}

\begin{table}[h]
\centering
\resizebox{\textwidth}{!}{
\begin{tabular}{llrccccccccc}
\toprule
Calib & Method & Rate & GSM8K & MATH & IFEval & PopQA & MMLU & BBH & TQA & LCB & Mean \\
\midrule
 & Base & & 92.17{\tiny$\pm$0.36} & 41.54{\tiny$\pm$0.37} & 21.21{\tiny$\pm$1.71} & 0.35{\tiny$\pm$0.00} & 65.15{\tiny$\pm$0.31} & 40.34{\tiny$\pm$0.88} & 65.92{\tiny$\pm$0.26} & 64.30{\tiny$\pm$0.61} & 48.87{\tiny$\pm$0.56} \\
\midrule
\multirow{30}{*}{Evol-CA} & HOPE & 10\% & 94.44{\tiny$\pm$0.36} & 42.61{\tiny$\pm$0.83} & 19.39{\tiny$\pm$0.86} & \textbf{0.35{\tiny$\pm$0.00}} & \textbf{63.87{\tiny$\pm$0.25}} & \textbf{41.51{\tiny$\pm$0.83}} & 63.13{\tiny$\pm$0.41} & 62.63{\tiny$\pm$0.66} & 48.49{\tiny$\pm$0.53} \\
 & REAP & 10\% & 93.43{\tiny$\pm$1.89} & 42.37{\tiny$\pm$1.15} & \textbf{20.00{\tiny$\pm$1.48}} & \textbf{0.35{\tiny$\pm$0.00}} & 62.74{\tiny$\pm$0.94} & 39.37{\tiny$\pm$0.36} & 64.53{\tiny$\pm$0.81} & 64.84{\tiny$\pm$1.84} & 48.45{\tiny$\pm$1.06} \\
 & MAN & 10\% & \textbf{95.20{\tiny$\pm$0.71}} & \textbf{42.71{\tiny$\pm$0.94}} & 18.79{\tiny$\pm$0.86} & \textbf{0.35{\tiny$\pm$0.00}} & 63.86{\tiny$\pm$0.45} & \textbf{41.51{\tiny$\pm$1.65}} & \textbf{64.64{\tiny$\pm$0.12}} & 63.44{\tiny$\pm$0.88} & \textbf{48.81{\tiny$\pm$0.70}} \\
 & EAN & 10\% & 94.44{\tiny$\pm$0.36} & 42.13{\tiny$\pm$1.47} & 18.79{\tiny$\pm$0.86} & \textbf{0.35{\tiny$\pm$0.00}} & 62.27{\tiny$\pm$0.81} & 39.88{\tiny$\pm$1.44} & 64.33{\tiny$\pm$0.48} & 63.98{\tiny$\pm$0.73} & 48.27{\tiny$\pm$0.77} \\
 & Freq & 10\% & 93.18{\tiny$\pm$0.62} & 42.64{\tiny$\pm$2.04} & 19.39{\tiny$\pm$1.71} & \textbf{0.35{\tiny$\pm$0.00}} & 62.94{\tiny$\pm$0.21} & 39.47{\tiny$\pm$0.38} & 63.39{\tiny$\pm$0.90} & \textbf{66.13{\tiny$\pm$1.52}} & 48.44{\tiny$\pm$0.92} \\
\addlinespace
 & HOPE & 20\% & \textbf{94.70{\tiny$\pm$0.62}} & 42.96{\tiny$\pm$0.59} & 18.18{\tiny$\pm$0.00} & \textbf{0.35{\tiny$\pm$0.00}} & 68.01{\tiny$\pm$0.86} & 39.62{\tiny$\pm$0.69} & 63.71{\tiny$\pm$0.16} & 62.96{\tiny$\pm$0.53} & 48.81{\tiny$\pm$0.43} \\
\bottomrule
\end{tabular}
}
\caption{\small Full Tulu results across all calibration sets and pruning rates for Qwen3.5-35B-A3B. Scores are multiplied by 100 (mean and standard deviation across 3 independent calibration trials). The best method (per condition) for each column is bolded.}
\label{tab:s1tulu_qwen3.5-35b}
\end{table}

\begin{table}[h]
\centering
\resizebox{\textwidth}{!}{
\begin{tabular}{llrccccccccc}
\toprule
Calib & Method & Rate & GSM8K & MATH & IFEval & PopQA & MMLU & BBH & TQA & LCB & Mean \\
\midrule
 & Base & & 94.19{\tiny$\pm$0.36} & 35.67{\tiny$\pm$1.29} & 19.39{\tiny$\pm$0.86} & 0.35{\tiny$\pm$0.00} & 26.68{\tiny$\pm$0.04} & 36.50{\tiny$\pm$0.45} & 59.36{\tiny$\pm$0.52} & 42.58{\tiny$\pm$0.70} & 39.34{\tiny$\pm$0.53} \\
\midrule
\multirow{30}{*}{Evol-CA} & HOPE & 10\% & 94.70{\tiny$\pm$0.00} & \textbf{40.08{\tiny$\pm$0.29}} & 18.18{\tiny$\pm$0.00} & \textbf{0.35{\tiny$\pm$0.00}} & 26.57{\tiny$\pm$0.27} & \textbf{38.55{\tiny$\pm$1.59}} & 60.62{\tiny$\pm$0.15} & \textbf{42.10{\tiny$\pm$0.95}} & \textbf{40.14{\tiny$\pm$0.41}} \\
 & REAP & 10\% & \textbf{94.95{\tiny$\pm$0.36}} & 38.43{\tiny$\pm$0.51} & \textbf{19.39{\tiny$\pm$1.71}} & \textbf{0.35{\tiny$\pm$0.00}} & 27.59{\tiny$\pm$0.43} & 37.22{\tiny$\pm$0.51} & 58.41{\tiny$\pm$0.69} & 41.08{\tiny$\pm$1.33} & 39.68{\tiny$\pm$0.69} \\
 & MAN & 10\% & 94.70{\tiny$\pm$0.00} & 36.03{\tiny$\pm$1.63} & 18.18{\tiny$\pm$0.00} & \textbf{0.35{\tiny$\pm$0.00}} & \textbf{27.69{\tiny$\pm$0.16}} & 36.61{\tiny$\pm$0.91} & 59.99{\tiny$\pm$0.18} & 39.41{\tiny$\pm$0.73} & 39.12{\tiny$\pm$0.45} \\
 & EAN & 10\% & 94.44{\tiny$\pm$0.36} & 36.53{\tiny$\pm$0.54} & 18.18{\tiny$\pm$1.48} & \textbf{0.35{\tiny$\pm$0.00}} & 26.97{\tiny$\pm$0.49} & 36.76{\tiny$\pm$0.71} & 62.24{\tiny$\pm$0.24} & 40.11{\tiny$\pm$3.80} & 39.45{\tiny$\pm$0.95} \\
 & Freq & 10\% & 94.19{\tiny$\pm$0.36} & 36.54{\tiny$\pm$0.69} & 18.79{\tiny$\pm$0.86} & \textbf{0.35{\tiny$\pm$0.00}} & 26.07{\tiny$\pm$0.64} & 35.89{\tiny$\pm$0.75} & \textbf{62.52{\tiny$\pm$0.50}} & 37.96{\tiny$\pm$0.66} & 39.04{\tiny$\pm$0.56} \\
\addlinespace
 & HOPE & 20\% & \textbf{95.71{\tiny$\pm$0.71}} & 38.49{\tiny$\pm$1.58} & 18.79{\tiny$\pm$0.86} & \textbf{0.35{\tiny$\pm$0.00}} & 25.20{\tiny$\pm$0.26} & \textbf{38.19{\tiny$\pm$0.70}} & 57.17{\tiny$\pm$0.41} & 36.29{\tiny$\pm$0.95} & 38.77{\tiny$\pm$0.68} \\
\bottomrule
\end{tabular}
}
\caption{\small Full Tulu results across all calibration sets and pruning rates for Qwen3.5-122B-A10B. Scores are multiplied by 100 (mean and standard deviation across 3 independent calibration trials). The best method (per condition) for each column is bolded.}
\label{tab:s1tulu_qwen3.5-122b}
\end{table}

\begin{table}[h]
\centering
\resizebox{\textwidth}{!}{
\begin{tabular}{llrccccccccc}
\toprule
Calib & Method & Rate & GSM8K & MATH & IFEval & PopQA & MMLU & BBH & TQA & LCB & Mean \\
\midrule
 & Base & & 90.91{\tiny$\pm$1.24} & 33.65{\tiny$\pm$1.49} & 29.09{\tiny$\pm$1.48} & 30.20{\tiny$\pm$0.52} & 32.51{\tiny$\pm$0.21} & 40.70{\tiny$\pm$1.18} & 58.86{\tiny$\pm$0.24} & 58.71{\tiny$\pm$0.26} & 46.83{\tiny$\pm$0.83} \\
\midrule
\multirow{30}{*}{Evol-CA} & HOPE & 10\% & 88.89{\tiny$\pm$0.71} & 32.89{\tiny$\pm$0.93} & 27.88{\tiny$\pm$3.09} & 20.56{\tiny$\pm$0.22} & 30.47{\tiny$\pm$0.03} & \textbf{41.97{\tiny$\pm$0.81}} & 58.51{\tiny$\pm$0.58} & 56.56{\tiny$\pm$0.33} & 44.72{\tiny$\pm$0.84} \\
 & REAP & 10\% & 89.90{\tiny$\pm$0.94} & 32.25{\tiny$\pm$1.53} & 27.88{\tiny$\pm$1.71} & 16.16{\tiny$\pm$0.23} & \textbf{32.21{\tiny$\pm$0.29}} & 39.57{\tiny$\pm$0.66} & 58.94{\tiny$\pm$0.09} & 57.31{\tiny$\pm$0.68} & 44.28{\tiny$\pm$0.77} \\
 & MAN & 10\% & 90.66{\tiny$\pm$1.29} & 32.70{\tiny$\pm$0.78} & 29.09{\tiny$\pm$0.00} & 18.50{\tiny$\pm$0.47} & 32.04{\tiny$\pm$0.25} & 41.10{\tiny$\pm$0.66} & 56.22{\tiny$\pm$0.94} & \textbf{59.73{\tiny$\pm$1.41}} & \textbf{45.01{\tiny$\pm$0.72}} \\
 & EAN & 10\% & \textbf{90.91{\tiny$\pm$1.24}} & \textbf{33.73{\tiny$\pm$0.88}} & 27.88{\tiny$\pm$0.86} & 19.83{\tiny$\pm$0.51} & 31.79{\tiny$\pm$0.21} & 40.13{\tiny$\pm$0.56} & 57.70{\tiny$\pm$0.42} & 56.61{\tiny$\pm$1.05} & 44.82{\tiny$\pm$0.72} \\
 & Freq & 10\% & 85.35{\tiny$\pm$0.94} & 28.34{\tiny$\pm$0.67} & \textbf{29.70{\tiny$\pm$0.86}} & \textbf{21.37{\tiny$\pm$0.40}} & 30.39{\tiny$\pm$0.36} & 40.24{\tiny$\pm$1.34} & \textbf{58.95{\tiny$\pm$0.43}} & 51.02{\tiny$\pm$0.97} & 43.17{\tiny$\pm$0.75} \\
\addlinespace
 & HOPE & 20\% & 88.89{\tiny$\pm$0.71} & 33.96{\tiny$\pm$0.49} & 24.85{\tiny$\pm$2.27} & 8.67{\tiny$\pm$0.27} & 30.83{\tiny$\pm$0.38} & 40.08{\tiny$\pm$0.62} & 53.76{\tiny$\pm$0.55} & 56.18{\tiny$\pm$2.27} & 42.15{\tiny$\pm$0.94} \\
\bottomrule
\end{tabular}
}
\caption{\small Full Tulu results across all calibration sets and pruning rates for GLM-4.5-Air. Scores are multiplied by 100 (mean and standard deviation across 3 independent calibration trials). The best method (per condition) for each column is bolded.}
\label{tab:s1tulu_glm4.5}
\end{table}

\begin{table}[h]
\centering
\begin{tabular}{lcccccc}
\toprule
Method & \multicolumn{3}{c}{Evol-CA} & \multicolumn{3}{c}{SB-verif} \\
\cmidrule(lr){2-4}
\cmidrule(lr){5-7}
 & 10\% & 25\% & 50\% & 10\% & 25\% & 50\% \\
\midrule
Base & 41.82 & 41.82 & 41.82 & 41.82 & 41.82 & 41.82 \\
\midrule
HOPE & 43.18{\tiny$\pm$1.40} & \textbf{44.58{\tiny$\pm$1.42}} & \textbf{42.33{\tiny$\pm$1.23}} & \textbf{43.01{\tiny$\pm$0.22}} & 42.02{\tiny$\pm$1.03} & \textbf{40.78{\tiny$\pm$0.86}} \\
REAP & 41.45{\tiny$\pm$0.59} & 42.74{\tiny$\pm$0.55} & 36.23{\tiny$\pm$0.93} & 42.43{\tiny$\pm$1.23} & 41.66{\tiny$\pm$0.64} & 40.21{\tiny$\pm$0.43} \\
MAN & \textbf{43.31{\tiny$\pm$0.94}} & 42.68{\tiny$\pm$1.69} & 37.54{\tiny$\pm$1.87} & 42.44{\tiny$\pm$1.58} & \textbf{42.52{\tiny$\pm$0.95}} & 40.62{\tiny$\pm$1.19} \\
EAN & 41.51{\tiny$\pm$0.89} & 42.88{\tiny$\pm$1.56} & 40.56{\tiny$\pm$0.67} & 41.76{\tiny$\pm$0.46} & 38.99{\tiny$\pm$0.87} & 36.24{\tiny$\pm$0.45} \\
Freq & 42.79{\tiny$\pm$0.82} & 42.54{\tiny$\pm$1.31} & 34.06{\tiny$\pm$0.26} & 38.80{\tiny$\pm$1.46} & 37.71{\tiny$\pm$0.18} & 30.16{\tiny$\pm$0.94} \\
\bottomrule
\end{tabular}
\caption{\small Full SWE-Bench Pro results across all calibration sets and pruning rates for Qwen3.5-35B-A3B. Scores are multiplied by 100 (mean and standard deviation across 3 independent calibration trials). The best method (per condition) for each column is bolded.}
\label{tab:s1swebenchpro_qwen3.5-35b}
\end{table}

\begin{table}[h]
\centering
\begin{tabular}{lcccccc}
\toprule
Method & \multicolumn{3}{c}{Evol-CA} & \multicolumn{3}{c}{SB-verif} \\
\cmidrule(lr){2-4}
\cmidrule(lr){5-7}
 & 10\% & 25\% & 50\% & 10\% & 25\% & 50\% \\
\midrule
Base & 50.16 & 50.16 & 50.16 & 50.16 & 50.16 & 50.16 \\
\midrule
HOPE & \textbf{52.44{\tiny$\pm$0.68}} & \textbf{52.29{\tiny$\pm$0.54}} & 43.84{\tiny$\pm$1.62} & \textbf{53.22{\tiny$\pm$0.66}} & 51.39{\tiny$\pm$0.47} & 50.63{\tiny$\pm$0.21} \\
REAP & 49.28{\tiny$\pm$0.46} & 47.17{\tiny$\pm$0.81} & 39.61{\tiny$\pm$0.27} & 51.94{\tiny$\pm$1.26} & 51.55{\tiny$\pm$0.18} & 46.78{\tiny$\pm$0.68} \\
MAN & 52.17{\tiny$\pm$0.34} & 47.80{\tiny$\pm$0.55} & 41.81{\tiny$\pm$0.18} & 51.11{\tiny$\pm$0.55} & \textbf{51.81{\tiny$\pm$0.30}} & 44.92{\tiny$\pm$1.32} \\
EAN & 52.10{\tiny$\pm$0.44} & 51.55{\tiny$\pm$0.27} & \textbf{48.32{\tiny$\pm$1.62}} & 50.66{\tiny$\pm$0.83} & 49.62{\tiny$\pm$0.42} & \textbf{51.80{\tiny$\pm$0.27}} \\
Freq & 51.47{\tiny$\pm$0.59} & 52.00{\tiny$\pm$0.77} & 45.71{\tiny$\pm$1.10} & 50.29{\tiny$\pm$1.25} & 50.87{\tiny$\pm$1.68} & 45.62{\tiny$\pm$0.92} \\
\bottomrule
\end{tabular}
\caption{\small Full SWE-Bench Pro results across all calibration sets and pruning rates for Qwen3.5-122B-A10B. Scores are multiplied by 100 (mean and standard deviation across 3 independent calibration trials). The best method (per condition) for each column is bolded.}
\label{tab:s1swebenchpro_qwen3.5-122b}
\end{table}

\begin{table}[h]
\centering
\begin{tabular}{lcccccc}
\toprule
Method & \multicolumn{3}{c}{Evol-CA} & \multicolumn{3}{c}{SB-verif} \\
\cmidrule(lr){2-4}
\cmidrule(lr){5-7}
 & 10\% & 25\% & 50\% & 10\% & 25\% & 50\% \\
\midrule
Base & 28.70 & 28.70 & 28.70 & 28.70 & 28.70 & 28.70 \\
\midrule
HOPE & 32.94{\tiny$\pm$1.05} & 32.80{\tiny$\pm$0.66} & \textbf{28.89{\tiny$\pm$0.76}} & 34.64{\tiny$\pm$0.54} & 33.26{\tiny$\pm$0.96} & \textbf{28.94{\tiny$\pm$1.21}} \\
REAP & 33.23{\tiny$\pm$0.82} & \textbf{34.06{\tiny$\pm$0.37}} & 26.88{\tiny$\pm$1.20} & 34.46{\tiny$\pm$0.59} & 33.44{\tiny$\pm$0.40} & 28.70{\tiny$\pm$1.40} \\
MAN & 33.93{\tiny$\pm$1.72} & 32.91{\tiny$\pm$1.32} & 23.65{\tiny$\pm$0.90} & \textbf{34.70{\tiny$\pm$1.07}} & \textbf{33.46{\tiny$\pm$0.90}} & 26.73{\tiny$\pm$0.91} \\
EAN & \textbf{34.09{\tiny$\pm$0.22}} & 32.62{\tiny$\pm$0.48} & 20.88{\tiny$\pm$1.51} & 33.41{\tiny$\pm$1.11} & 32.71{\tiny$\pm$0.51} & 26.11{\tiny$\pm$0.84} \\
Freq & 27.40{\tiny$\pm$0.54} & 1.14{\tiny$\pm$0.62} & 0.00{\tiny$\pm$0.00} & 30.53{\tiny$\pm$0.60} & 22.47{\tiny$\pm$0.92} & 1.69{\tiny$\pm$0.13} \\
\bottomrule
\end{tabular}
\caption{\small Full SWE-Bench Pro results across all calibration sets and pruning rates for GLM-4.5-Air. Scores are multiplied by 100 (mean and standard deviation across 3 independent calibration trials). The best method (per condition) for each column is bolded.}
\label{tab:s1swebenchpro_glm4.5}
\end{table}

\begin{figure}[h]
\centering
\includegraphics[width=\columnwidth]{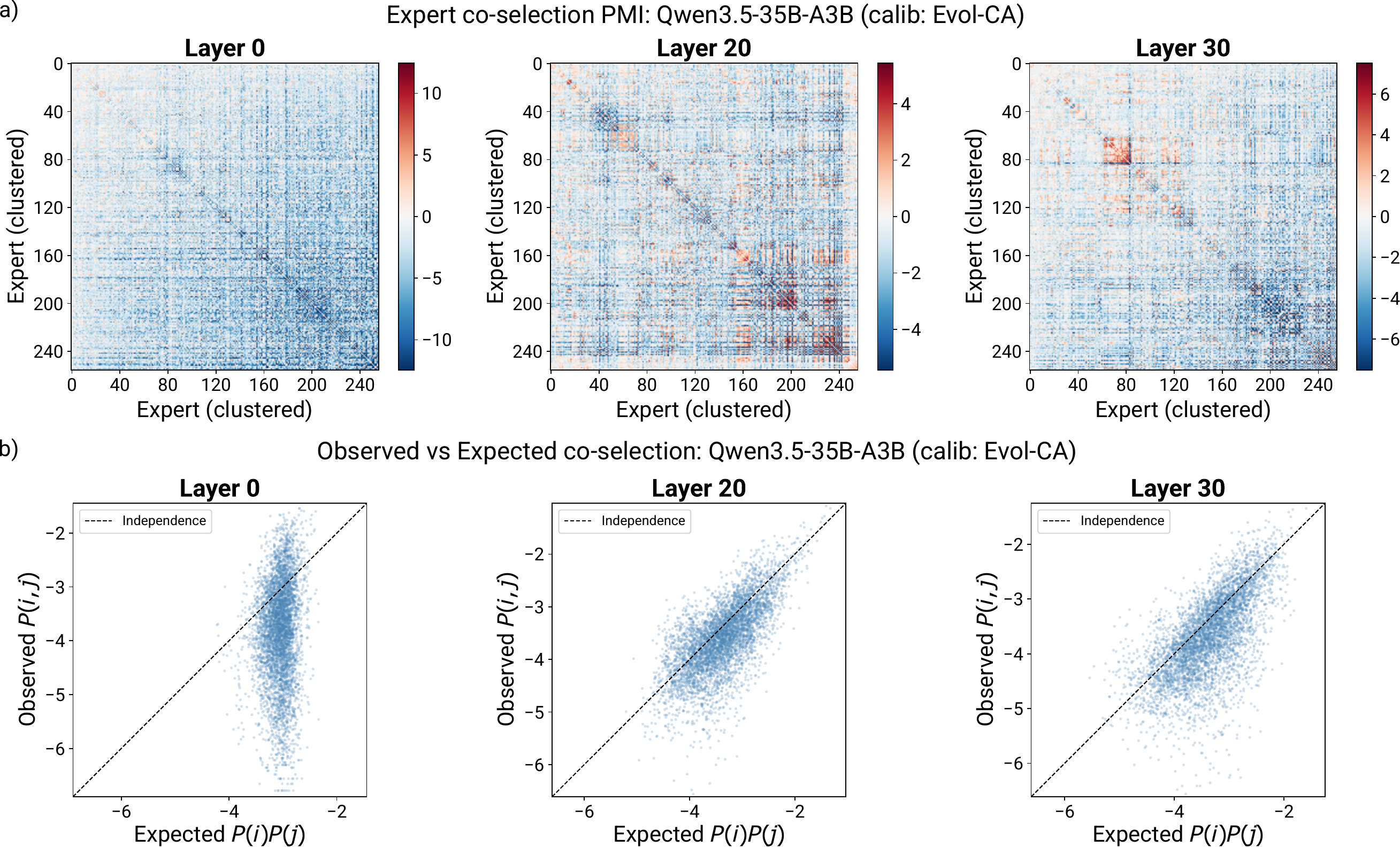}
\caption{\small Expert co-selection is highly non-random, which motivates interaction-aware pruning. \textbf{a)} Pointwise mutual information (PMI) between expert pairs. PMI is defined between two experts $i,j$ (within a layer) as $PMI(i,j)=log_{2}\frac{P(i,j)}{P(i)P(j)}$, where $P(i,j)$ is the fraction of tokens which activate both experts, and $P(i)$ is the fraction of tokens which activate only one expert $i$. We show PMI between experts at three different layers of Qwen3.5-35B-A3B. Positive values (red) indicate experts are co-selected more often than expected due to independent chance; negative values (blue) indicate experts are co-selected less often (mutual exclusion). \textbf{b)} Observed co-selection probability $P(i,j)$ versus independence expectation $P(i)P(j)$. Under completely independent routing, all points should lie on the diagonal. Any substantial deviation (which is especially pronounced at layer 0) demonstrates that experts operate in cooperative groups.}
\label{fig:S2pmi}
\end{figure}

\begin{figure}[h]
\centering
\includegraphics[width=\columnwidth]{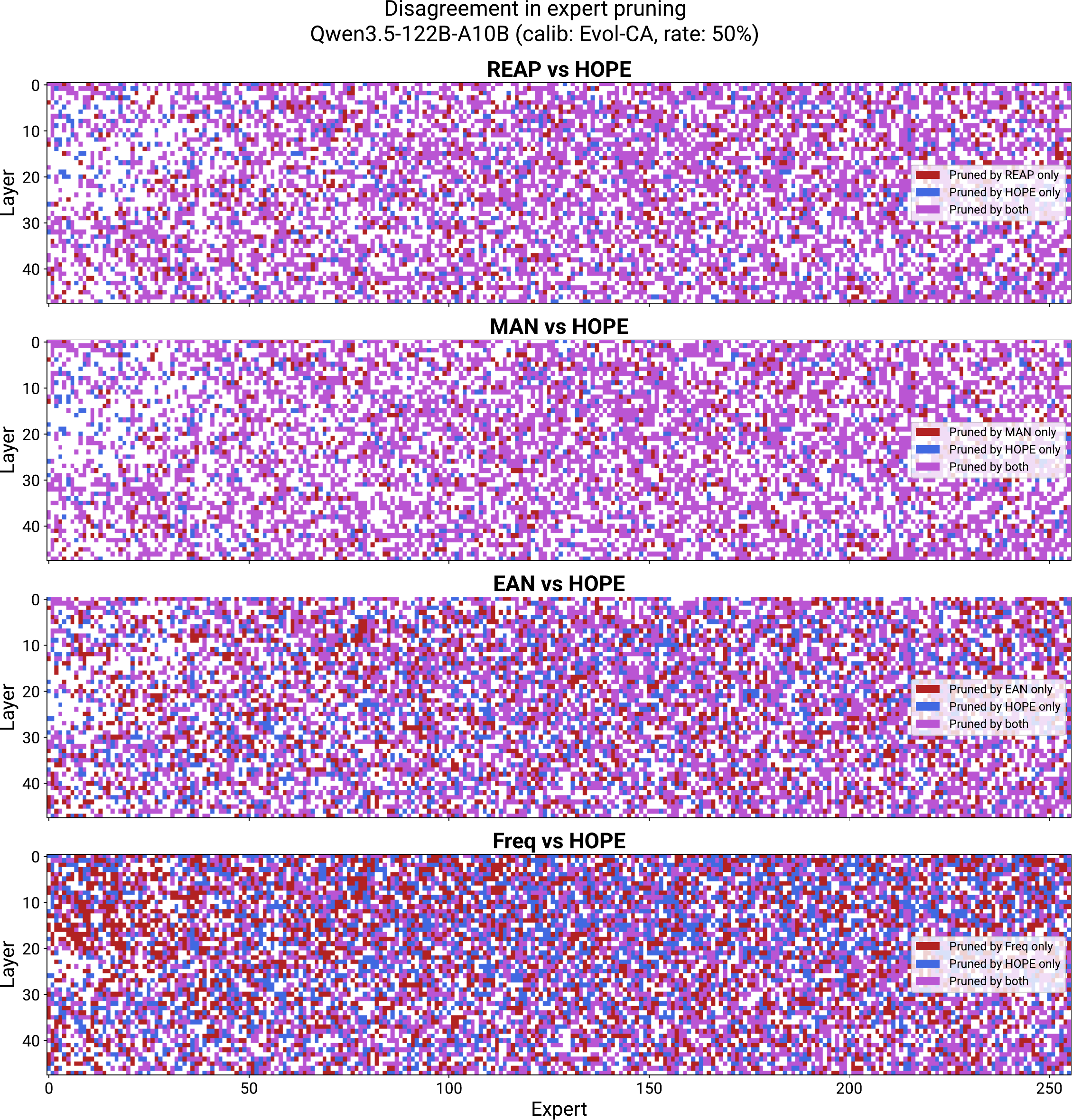}
\caption{\small Spatial visualization of disagreement in expert pruning between HOPE and each baseline (Qwen3.5-122B-A10B). Each panel shows an $L\times E$ grid. Disagreement is typically distributed across all layers and experts rather than purely concentrated in certain regions.}
\label{fig:S3blacklist_overlap_grid}
\end{figure}

\begin{table}[h]
\centering
\begin{tabular}{llcccc}
\toprule
 & & \multicolumn{2}{c}{Calibration} & & \\
\cmidrule(lr){3-4}
Model & Method & Time/token (ms) & 1k prompts (hr) & QP solve (s) & Ratio \\
\midrule
\multirow{2}{*}{\shortstack{Qwen-35B\\($L=40,E=256$)}} & HOPE & 2.68 & 1.27 & 75.4 & \multirow{2}{*}{$1.077\times$} \\
 & REAP & 2.53 & 1.20 & N/A & \\
\midrule
\multirow{2}{*}{\shortstack{Qwen-122B\\($L=48,E=256$)}} & HOPE & 3.47 & 1.55 & 92.4 & \multirow{2}{*}{$1.067\times$} \\
 & REAP & 3.31 & 1.48 & N/A & \\
\midrule
\bottomrule
\end{tabular}
\caption{\small Computational cost of HOPE vs REAP. Calibration time is given in milliseconds per token, and in hours per 1000 prompts (given a random sample of Evol-CodeAlpaca). QP solve time is the time taken to solve the HOPE-specific quadratic program over all layers total (a one-time cost). Ratio is the total amount of time taken for HOPE to make its pruning decisions (including QP solve time) over the time taken for REAP.}
\label{tab:S2timing}
\end{table}

\begin{figure}[h]
\centering
\includegraphics[width=0.4\columnwidth]{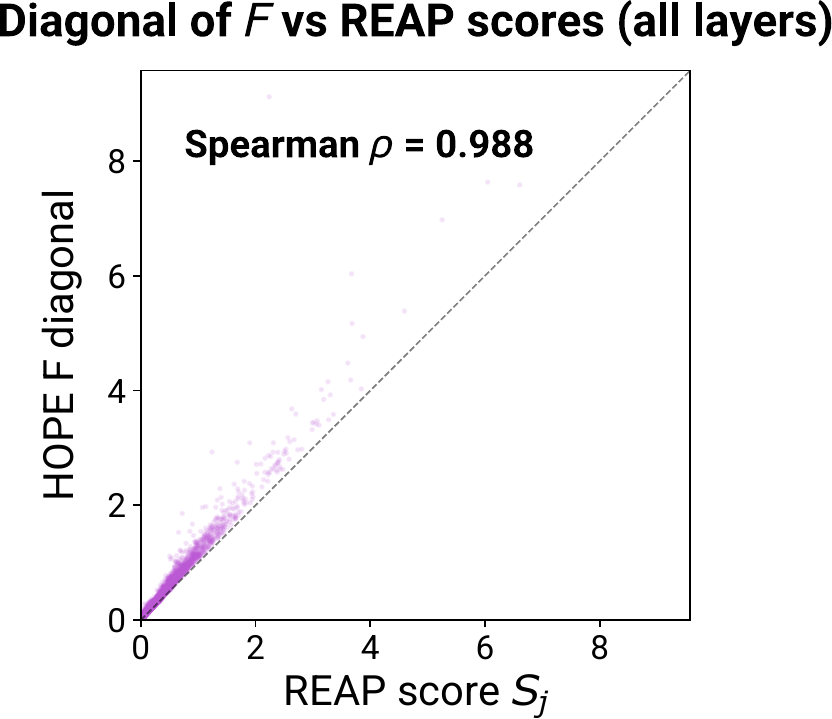}
\caption{\small The diagonal of the interaction matrix $F$ recovers REAP’s expert rankings. Each point represents one expert (pooled across layers of Qwen3.5-122B-A10B, calib: Evol-CA). The x-axis is the REAP score $S^{\text{REAP}}_{k}$ and the y-axis is $\sqrt{F_{k,k}}$. The Spearman correlation is near perfect ($\rho = 0.988$), confirming that REAP is a special case of HOPE.}
\label{fig:S4fdiag_vs_reap}
\end{figure}

\begin{figure}[h]
\centering
\includegraphics[width=\columnwidth]{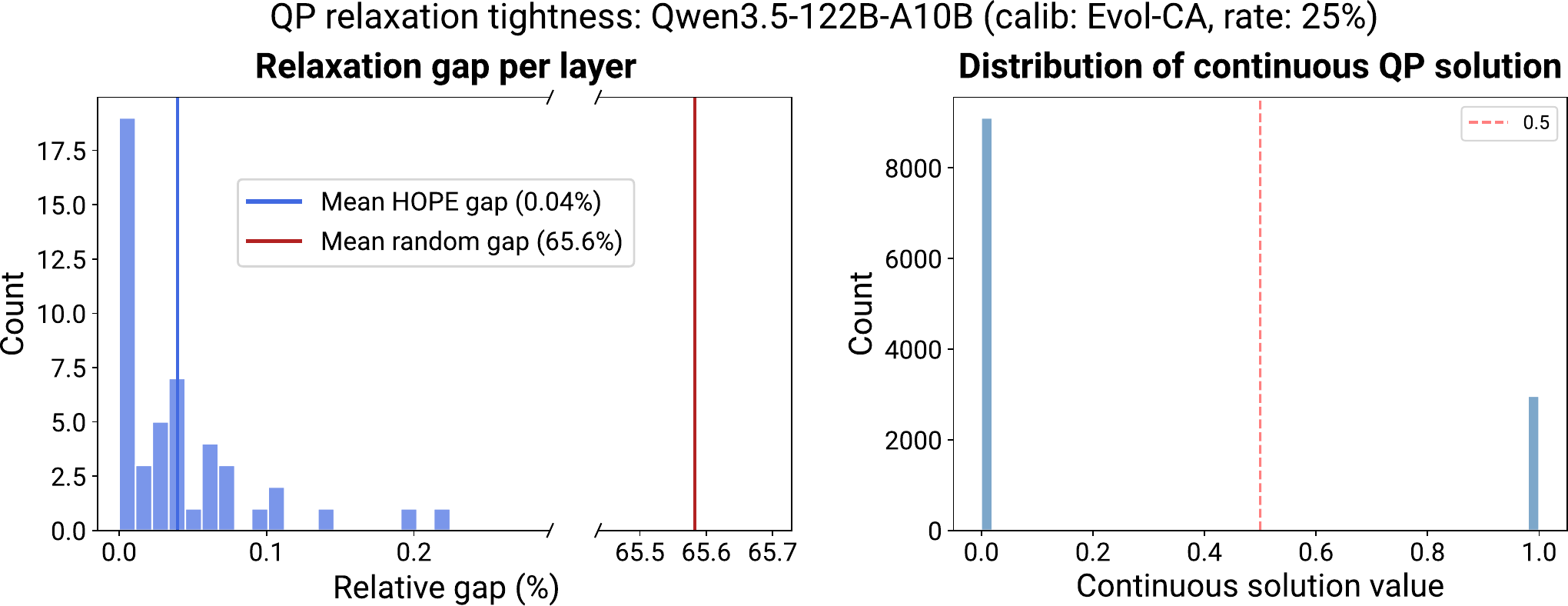}
\caption{\small HOPE’s continuous relaxation to solve the quadratic program is tight. \textbf{a)} Histogram of the relative gap in the binary objective value versus the continuous objective value (measured as $\frac{\text{obj}_{bin} - \text{obj}_{cont}}{\text{obj}_{bin}}$) per layer. The mean gap is only 0.04\%, compared to an expected gap of 65.6\% for randomly sampled feasible binary solutions (red line). \textbf{b)} Distribution of continuous solution values in $p$ across all layers and experts. The overwhelming concentration of continuous solutions near 0 and 1 confirms that the relaxation produces effectively binary solutions, thereby rounding is nearly lossless.}
\label{fig:S5qp_relaxation}
\end{figure}

\begin{figure}[h]
\centering
\includegraphics[width=0.6\columnwidth]{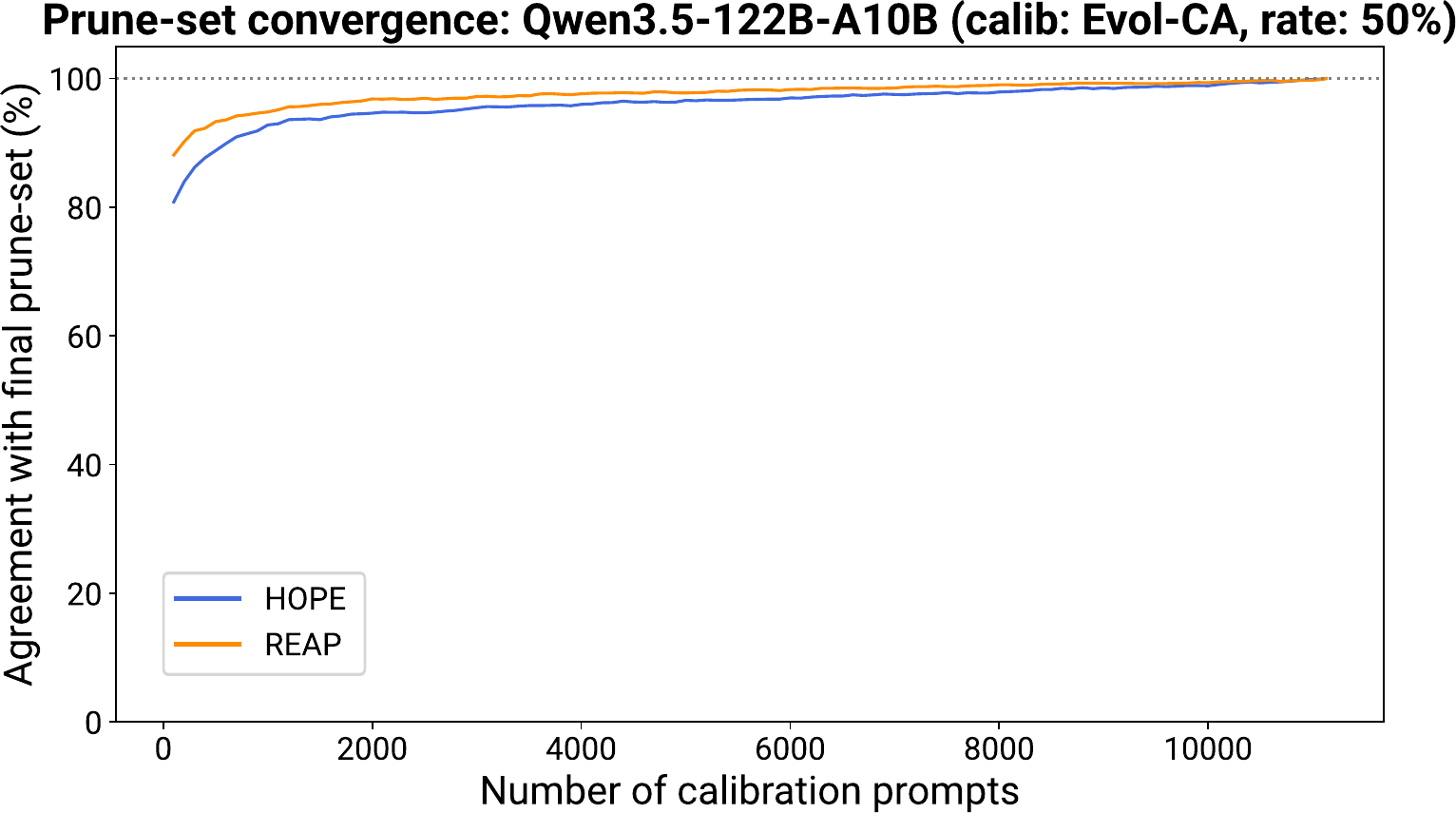}
\caption{\small Convergence of pruning decisions with calibration data. We measure the fraction of pruning decisions that agree with the final pruning set (computed from the entire calibration set) as a function of the number of calibration prompts processed. Both HOPE and REAP stabilize within 10k prompts, with HOPE requiring marginally more data to converge due to estimating pairwise (rather than per-expert) statistics. Convergence is shown for 50\% pruning on Qwen3.5-35B-A3B with the Evol-CodeAlpaca calibration set.}
\label{fig:S6f_convergence}
\end{figure}

\clearpage

\section{Cross-layer HOPE (CHOPE)}
\label{app:chope}

In this section, we present an extension of HOPE which accounts for expert interactions across different layers, which enables non-uniform layer-wise pruning budgets (i.e. optimizing for a different number of pruned experts in each layer). We call this extension CHOPE (\textbf{C}ross-layer \textbf{H}igher-\textbf{O}rder \textbf{P}runing of \textbf{E}xperts).

\subsection{Setup}

Consider an $L$-layer MoE transformer, where each layer $\ell$ consists of attention followed by an MoE layer, each with residual connections:
\begin{align*}
  z^{\ell} &= x^{\ell-1} + \text{Att}^{\ell}(x^{\ell-1}) \\
  x^{\ell} &= z^{\ell} + \text{MoE}^{\ell}(z^{\ell})
\end{align*}
where $x^{\ell-1} \in \mathbb{R}^{d}$ is the residual stream entering layer $\ell$, and $z^{\ell} \in \mathbb{R}^{d}$ is the post-attention representation fed to the MoE. Layer norms are absorbed into the respective $\text{Att},\text{MoE}$ functions. Note that in this derivation, we assume each layer consists of attention followed by an MoE block, but the attention can be equivalently replaced by other architectures (e.g. an SSM).

A forward pass computes $x^{0}, z^{1}, x^{1}, \ldots, x^{L}$ in sequence, where $x^{0}$ is the input token embedding and $x^{L}$ is fed to the language-model head.

At each layer $\ell$, the MoE contains $E$ experts with output $\text{MoE}^{\ell}(z^{\ell}) = \sum\limits_{k=1}^{E} g_{k}^{\ell}(z^{\ell})\,f_{k}^{\ell}(z^{\ell})$, using the same notation as in Appendix~\ref{app:hope_deriv}. We use top-$K$ routing with selected set $T^{\ell}(z^{\ell}) \subseteq \{1,\ldots,E\}$.

At each layer $\ell$, we prune a set of experts $P^{\ell} \subseteq \{1,\ldots,E\}$. Upon pruning, for any token, the pruned experts $P^{\ell} \cap T^{\ell}(z^{\ell})$ are replaced by the next-highest-logit experts $R^{\ell}(z^{\ell})$, with $\vert R^{\ell}(z^{\ell})\vert = \vert P^{\ell} \cap T^{\ell}(z^{\ell})\vert$. The pruned MoE output uses new gate weights ${g'}_{k}^{\ell}(z^{\ell})$. Note that this is the same setup as in Appendix~\ref{app:hope_deriv}, tracking layer indices.

\subsection{Error propagation}

When we prune experts, the effects of pruning propagate through the network from early to later layers. After pruning experts from each layer, the forward pass computes perturbed quantities $\tilde{x}^{0}, \tilde{z}^{1}, \tilde{x}^{1}, \ldots, \tilde{x}^{L}$ (with $\tilde{x}^{0} = x^{0}$). Define the output \textit{error} at layer $\ell$ as $\epsilon^{\ell} = \tilde{x}^{\ell} - x^{\ell}$. Here, we will quantify the error $\epsilon^{\ell}$ in each layer.

\textbf{Error through attention}: At layer $\ell$, the perturbed post-attention representation is $\tilde{z}^{\ell} = \tilde{x}^{\ell-1} + \text{Att}^{\ell}(\tilde{x}^{\ell-1})$. The error in the output of this attention layer is:
\[
\tilde{z}^{\ell} - z^{\ell} = (\tilde{x}^{\ell-1} - x^{\ell-1}) + (\text{Att}^{\ell}(\tilde{x}^{\ell-1}) - \text{Att}^{\ell}(x^{\ell-1}))
\]

Since no pruning/modification happens at the attention block, its error $\tilde{z}^{\ell} - z^{\ell}$ arises purely from error in the input being propagated through.

To quantify $\tilde{z}^{\ell} - z^{\ell}$, we perform a first-order Taylor expansion of $\text{Att}^{\ell}$ around $x^{\ell-1}$:
\[
\text{Att}^{\ell}(\tilde{x}^{\ell-1}) \approx \text{Att}^{\ell}(x^{\ell-1}) + J_{A}^{\ell}(\tilde{x}^{\ell-1} - x^{\ell-1})
\]
where $J_{A}^{\ell} = \frac{\partial\,\text{Att}^{\ell}(x)}{\partial x}\big\vert_{x = x^{\ell-1}} \in \mathbb{R}^{d \times d}$ is the Jacobian of the attention output with respect to its input, evaluated at the unperturbed input.

Substituting, we have:
\[
\tilde{z}^{\ell} - z^{\ell} \approx (I + J_{A}^{\ell})\,\epsilon^{\ell-1}
\]

\textbf{Error through MoE}: Similar to error propagation through attention, the error in the output of the MoE is:
\[
\tilde{x}^{\ell} - x^{\ell} = (\tilde{z}^{\ell} - z^{\ell}) + (\widetilde{\text{MoE}}^{\ell}(\tilde{z}^{\ell}) - \text{MoE}^{\ell}(z^{\ell}))
\]
where $\widetilde{\text{MoE}}^{\ell}$ is the MoE block after pruning. Unlike the attention block, the error in the output of the MoE block arises from both propagated error \textit{and} the error accrued from pruning experts within the block.

Let $\Delta^{\ell}(\tilde{z}^{\ell}) = \widetilde{\text{MoE}}^{\ell}(\tilde{z}^{\ell}) - \text{MoE}^{\ell}(\tilde{z}^{\ell})$ be the error resulting from pruning at layer $\ell$ (the difference in MoE outputs before and after pruning, applied to the perturbed input). Substituting, this gives us:
\[
\tilde{x}^{\ell} - x^{\ell} = (\tilde{z}^{\ell} - z^{\ell}) + (\text{MoE}^{\ell}(\tilde{z}^{\ell}) - \text{MoE}^{\ell}(z^{\ell})) + \Delta^{\ell}(\tilde{z}^{\ell})
\]

We perform a first-order Taylor approximation of the unpruned MoE around $z^{\ell}$ (like with attention), yielding:
\[
\epsilon^{\ell} = \tilde{x}^{\ell} - x^{\ell} \approx (I + J_{M}^{\ell})\,(\tilde{z}^{\ell} - z^{\ell}) + \Delta^{\ell}(\tilde{z}^{\ell})
\]
where $J_{M}^{\ell} = \frac{\partial\,\text{MoE}^{\ell}(z)}{\partial z}\big\vert_{z = z^{\ell}}$ is the MoE Jacobian (evaluated at the unperturbed input).

\textbf{Closed form of the error}: Substituting in the attention error $\tilde{z}^{\ell} - z^{\ell}$ into the MoE error $\epsilon^{\ell}$:
\[
\epsilon^{\ell} = (I + J_{M}^{\ell})(I + J_{A}^{\ell})\,\epsilon^{\ell-1} + \Delta^{\ell}(\tilde{z}^{\ell})
\]
with $\epsilon^{0} = 0$.

This recurrence has the following closed form:
\[
\epsilon^{L} = \sum\limits_{j=1}^{L} \Phi_{L,j}\,\Delta^{j}(\tilde{z}^{j})
\]
where $\Phi_{\ell,j}$ is the \textit{error propagation operator}, defined as follows:
\[
\Phi_{\ell,j} = \prod\limits_{m=j+1}^{\ell} (I + J_{M}^{m})(I + J_{A}^{m}), \quad\text{with}\quad \Phi_{\ell,\ell} = I
\]
$\Phi_{\ell,j}$ quantifies how error introduced at layer $j$ is amplified or attenuated as it propagates to the output at layer $\ell$.

Note that $\epsilon^{L}$ is the overall error of the last layer of the network, which includes error arising from expert pruning in all MoE layers, as well as the error propagating from earlier to later layers.

\textbf{Decoupling error from $\tilde{z}^{j}$}: Computing $\Delta^{j}(\tilde{z}^{j})$ exactly requires knowing the perturbed input $\tilde{z}^{j}$, which itself depends on all previous pruning decisions. This makes optimizing this error over prune-sets intractable. Instead, we approximate $\Delta^{j}(\tilde{z}^{j}) \approx \Delta^{j}(z^{j})$. That is, the pruning error at layer $j$ is well-approximated by the error computed on the clean (unpruned) input. This decouples the per-layer pruning errors from the propagation operators:
\[
\epsilon^{L} \approx \sum\limits_{j=1}^{L} \Phi_{L,j}\,\Delta^{j}(z^{j})
\]

\subsection{Upper-bounding the substitution error}

Here, we will bound the substitution-error component of $\epsilon^{L}$. The proof follows analogously to that in Appendix~\ref{app:hope_deriv}.

As with per-layer HOPE (Appendix~\ref{app:hope_deriv}), we decompose $\Delta^{j}(z^{j})$ into a substitution error and a renormalization error:
\begin{align*}
\Delta^{j}(z^{j}) &= \left[ \sum\limits_{i\in R^{j}(z^{j})}g_{i}^{'j}(z^{j})f_{i}^{j}(z^{j}) - \sum\limits_{k\in T^{j}(z^{j})\cap
P^{j}}g_{k}^{j}(z^{j})f_{k}^{j}(z^{j})\right] \\
&\quad+ \left[\sum\limits_{k\in T^{j}(z^{j}) \setminus P^{j}}(g_{k}^{'j}(z^{j})- g_{k}^{j}(z^{j}))f_{k}^{j}(z^{j})\right]
\end{align*}

Substituting this decomposition into $\epsilon^{L}$:
\begin{align*}
\epsilon^{L} &= \underbrace{\sum\limits_{j=1}^{L}\Phi^{L,j} \left[ \sum\limits_{i\in R^{j}(z^{j})}g_{i}^{'j}(z^{j})f_{i}^{j}(z^{j}) - \sum\limits_{k\in
T^{j}(z^{j})\cap P^{j}}g_{k}^{j}(z^{j})f_{k}^{j}(z^{j})\right]}_{\text{substitution error}} \\
&\quad+ \underbrace{\sum\limits_{j=1}^{L}\Phi^{L,j} \left[\sum\limits_{k\in T^{j}(z^{j}) \setminus P^{j}}(g_{k}^{'j}(z^{j})-
g_{k}^{j}(z^{j}))f_{k}^{j}(z^{j})\right]}_{\text{renormalization error}}
\end{align*}

The renormalization error remains small by the same argument as in Appendix~\ref{app:hope_deriv} (with an additional factor of $\Phi_{L,j}$). As before, we focus on minimizing the squared cross-layer substitution error:
\[
\biggl\Vert \sum\limits_{j=1}^{L} \Phi_{L,j} \biggl(\sum\limits_{k \in T^{j}(z^{j}) \cap P^{j}} g_{k}^{j}(z^{j})\,f_{k}^{j}(z^{j}) - \sum\limits_{i \in R^{j}(z^{j})} {g'}_{i}^{j}(z^{j})\,f_{i}^{j}(z^{j})\biggr) \biggr\Vert_{2}^{2}
\]

As before, we decompose this squared norm into three components ($A + B + C$):
\begin{align*}
A &= \biggl\Vert \sum\limits_{j=1}^{L} \Phi_{L,j} \sum\limits_{k \in T^{j}(z^{j}) \cap P^{j}} g_{k}^{j}(z^{j})\,f_{k}^{j}(z^{j}) \biggr\Vert_{2}^{2}
\\
B &= -2\biggl\langle \sum\limits_{j=1}^{L} \Phi_{L,j} \sum\limits_{k \in T^{j}(z^{j}) \cap P^{j}} g_{k}^{j}(z^{j})\,f_{k}^{j}(z^{j}),\;
\sum\limits_{j=1}^{L} \Phi_{L,j} \sum\limits_{i \in R^{j}(z^{j})} {g'}_{i}^{j}(z^{j})\,f_{i}^{j}(z^{j}) \biggr\rangle\\
C &= \biggl\Vert \sum\limits_{j=1}^{L} \Phi_{L,j} \sum\limits_{i \in R^{j}(z^{j})} {g'}_{i}^{j}(z^{j})\,f_{i}^{j}(z^{j}) \biggr\Vert_{2}^{2}
\end{align*}

By Cauchy--Schwarz, $\vert B\vert \leq 2\sqrt{A}\sqrt{C}$, so $A + B + C \leq (\sqrt{A} + \sqrt{C})^{2}$.

We now upper bound both $\sqrt{A}$ and $\sqrt{C}$ using a common quantity. Define the following:
\[
Z = \bigl(\sum\limits_{j=1}^{L} \sum\limits_{k \in T^{j}(z^{j}) \cap P^{j}} g_{k}^{j}(z^{j}) \,\Vert\Phi_{L,j}\,f_{k}^{j}(z^{j})\Vert_{2}\bigr)^{2}
\]

\textbf{Bounding} $\sqrt{A}$:
By the triangle inequality (applied first over layers, then over experts within each layer):
\[
\sqrt{A} = \biggl\Vert \sum\limits_{j=1}^{L} \Phi_{L,j} \sum\limits_{k \in T^{j}(z^{j}) \cap P^{j}} g_{k}^{j}(z^{j})\,f_{k}^{j}(z^{j}) \biggr\Vert_{2}
\leq \sum\limits_{j=1}^{L} \sum\limits_{k \in T^{j}(z^{j}) \cap P^{j}} g_{k}^{j}(z^{j}) \,\Vert\Phi_{L,j}\,f_{k}^{j}(z^{j})\Vert_{2} = \sqrt{Z}
\]

\textbf{Bounding} $\sqrt{C}$:
By the triangle inequality:
\[
\sqrt{C} = \biggl\Vert \sum\limits_{j=1}^{L} \Phi_{L,j} \sum\limits_{i \in R^{j}(z^{j})} {g'}_{i}^{j}(z^{j})\,f_{i}^{j}(z^{j}) \biggr\Vert_{2} \leq
\sum\limits_{j=1}^{L} \sum\limits_{i \in R^{j}(z^{j})} {g'}_{i}^{j}(z^{j}) \,\Vert\Phi_{L,j}\,f_{i}^{j}(z^{j})\Vert_{2}
\]

We have $\Vert\Phi_{L,j}\,f_{i}^{j}(z^{j})\Vert_{2} \leq \max_{j,k}\Vert\Phi_{L,j}\,f_{k}^{j}(z^{j})\Vert_{2}$ for each replacement expert $i$ in each layer $j$. Therefore:
\[
\sum\limits_{j=1}^{L} \sum\limits_{i \in R^{j}(z^{j})} {g'}_{i}^{j}(z^{j}) \Vert\Phi_{L,j}\,f_{i}^{j}(z^{j})\Vert_{2} \leq \max_{j,k}\Vert\Phi_{L,j}\,f_{k}^{j}(z^{j})\Vert_{2} \sum\limits_{j=1}^{L} \sum\limits_{i \in R^{j}(z^{j})} {g'}_{i}^{j}(z^{j})
\]

By Lemma~\ref{lem:gate-mass} (applied independently at each layer $j$), $\sum_{i \in R^{j}(z^{j})} {g'}_{i}^{j}(z^{j}) \leq \sum_{k \in T^{j}(z^{j}) \cap P^{j}} g_{k}^{j}(z^{j})$. Summing over layers:
\[
\max_{j,k}\Vert\Phi_{L,j}\,f_{k}^{j}(z^{j})\Vert_{2} \sum\limits_{j=1}^{L} \sum\limits_{i \in R^{j}(z^{j})} {g'}_{i}^{j}(z^{j}) \leq \max_{j,k}\Vert\Phi_{L,j}\,f_{k}^{j}(z^{j})\Vert_{2} \sum\limits_{j=1}^{L} \sum\limits_{k \in T^{j}(z^{j}) \cap P^{j}} g_{k}^{j}(z^{j})
\]

Furthermore, for each pruned expert $k$ in layer $j$, we have $g_{k}^{j}(z^{j}) \leq g_{k}^{j}(z^{j})\frac{\Vert\Phi_{L,j}\,f_{k}^{j}(z^{j})\Vert_{2}}{\min_{j,k}\Vert\Phi_{L,j}f_{k}^{j}(z^{j})\Vert_{2}}$. Substituting:
\[
\max_{j,k}\Vert\Phi_{L,j}\,f_{k}^{j}(z^{j})\Vert_{2} \sum\limits_{j=1}^{L} \sum\limits_{k \in T^{j}(z^{j}) \cap P^{j}} g_{k}^{j}(z^{j})
\]
\[
\leq \frac{\max_{j,k}\Vert\Phi_{L,j}\,f_{k}^{j}(z^{j})\Vert_{2}}{\min_{j,k}\Vert\Phi_{L,j}\,f_{k}^{j}(z^{j})\Vert_{2}} \sum\limits_{j=1}^{L} \sum\limits_{k \in
T^{j}(z^{j}) \cap P^{j}} g_{k}^{j}(z^{j})\Vert\Phi_{L,j}\,f_{k}^{j}(z^{j})\Vert_{2} = \rho_{\Phi}\sqrt{Z}
\]

Thus, $\sqrt{C} \leq \rho_{\Phi}\sqrt{Z}$.

\textbf{Together}:
\begin{align*}
\biggl\Vert \sum\limits_{j=1}^{L} \Phi_{L,j} \biggl(\sum\limits_{k \in T^{j}(z^{j}) \cap P^{j}} g_{k}^{j}(z^{j})\,f_{k}^{j}(z^{j}) - \sum\limits_{i \in
R^{j}(z^{j})} {g'}_{i}^{j}(z^{j})\,f_{i}^{j}(z^{j})\biggr) \biggr\Vert_{2}^{2} &\leq (\sqrt{A} + \sqrt{C})^{2} \\
&\leq (\sqrt{Z} + \rho_{\Phi}\sqrt{Z})^{2} \\
&= (1 + \rho_{\Phi})^{2} \cdot Z
\end{align*}

This gives us the following Theorem, analogous to Theorem~\ref{thm:z_bound}:

\begin{theorem}[Cross-layer upper bound]
For any collection of prune-sets $\{P^{1}, \ldots, P^{L}\}$, the squared cross-layer substitution error is bounded by:
\[
(1 + \rho_{\Phi})^{2} \cdot Z \quad\text{where}\quad \rho_{\Phi} = \frac{\max_{j,k} \Vert\Phi_{L,j}\,f_{k}^{j}(z^{j})\Vert_{2}}{\min_{j,k} \Vert\Phi_{L,j}\,f_{k}^{j}(z^{j})\Vert_{2}}
\]
\label{thm:chope_bound}
\end{theorem}

Note that $\rho_{\Phi}$ generalizes the per-layer $\rho$ from Theorem~\ref{thm:z_bound}: it accounts for how $\Phi_{L,j}$ differentially amplifies expert contributions from different layers. Because the ratio is independent of the prune-sets $\{P^{1},\ldots,P^{L}\}$, we can simply minimize $Z$.

\subsection{Minimizing $Z$ as a quadratic program}

Our goal is to minimize $\mathop{\mathbb{E}}\limits_{x \sim \mathcal{D}}[Z]$ over all prune-sets $\{P^{1}, \ldots, P^{L}\}$ jointly. We introduce binary decision variables $p_{k}^{j} \in \{0,1\}$, where $p_{k}^{j} = 1$ iff expert $k$ in layer $j$ is pruned.

Substituting:
\begin{align*}
Z &= \biggl(\sum\limits_{j=1}^{L} \sum\limits_{k=1}^{E} p_{k}^{j} \cdot \mathbf{1}[k \in T^{j}(z^{j})] \cdot g_{k}^{j}(z^{j})
\,\Vert\Phi_{L,j}\,f_{k}^{j}(z^{j})\Vert_{2}\biggr)^{2} \\
&= \sum\limits_{j=1}^{L} \sum\limits_{\ell=1}^{L} \sum\limits_{k=1}^{E} \sum\limits_{m=1}^{E} p_{k}^{j}\,p_{m}^{\ell} \cdot \mathbf{1}[k \in T^{j}(z^{j})]
\cdot \mathbf{1}[m \in T^{\ell}(z^{\ell})] \\
&\qquad\qquad \cdot\, g_{k}^{j}(z^{j})\,g_{m}^{\ell}(z^{\ell})
\,\Vert\Phi_{L,j}\,f_{k}^{j}(z^{j})\Vert_{2}\,\Vert\Phi_{L,\ell}\,f_{m}^{\ell}(z^{\ell})\Vert_{2}
\end{align*}

Taking the empirical expectation over the calibration set (using conditional normalization as in per-layer HOPE):
\[
\mathop{\mathbb{E}}\limits_{x}[Z] = \sum\limits_{j=1}^{L} \sum\limits_{\ell=1}^{L} \sum\limits_{k=1}^{E} \sum\limits_{m=1}^{E} p_{k}^{j}\,p_{m}^{\ell} \cdot G[(j,k),(\ell,m)] = p^{\top} G\, p
\]
where $p \in \{0,1\}^{LE}$ is the concatenated binary vector over all layers, and $G \in \mathbb{R}^{LE \times LE}$ is defined as:
\[
G[(j,k),(\ell,m)] = \frac{1}{\vert\mathcal{X}_{k,m}^{j,\ell}\vert} \sum\limits_{x \in \mathcal{X}_{k,m}^{j,\ell}} g_{k}^{j}(z^{j})\,g_{m}^{\ell}(z^{\ell})\,\Vert\Phi_{L,j}\,f_{k}^{j}(z^{j})\Vert_{2}\,\Vert\Phi_{L,\ell}\,f_{m}^{\ell}(z^{\ell})\Vert_{2}
\]
with $\mathcal{X}_{k,m}^{j,\ell} = \{x : k \in T^{j}(z^{j}) \text{ and } m \in T^{\ell}(z^{\ell})\}$ denoting the set of tokens which activate expert $k$ in layer $j$ and expert $m$ in layer $\ell$.

\begin{theorem}[Cross-layer QP]
Given a total pruning budget $B$ (the total number of experts to remove across all layers), the collection of prune-sets $\{P^{1},\ldots,P^{L}\}$ which minimizes $\mathop{\mathbb{E}}\limits_{x \in \mathcal{D}}[Z]$ is the solution to:
\[
p^{*} = \argmin_{p}\; p^{\top} G\, p \quad \text{s.t.} \quad p \in \{0,1\}^{LE},\; \sum\limits_{k=1}^{LE} p_{k} = B
\]
\label{thm:chope_qp}
\end{theorem}

\subsection{Relationship to per-layer HOPE}

The cross-layer matrix $G$ encodes both within-layer and across-layer expert interactions:
\begin{itemize}[leftmargin=1em]
    \item $E\times E$ blocks on the diagonal of $G$ (entries where $j = \ell$) correspond to the per-layer $F$-matrices of HOPE, scaled by $\Vert\Phi_{L,j}\Vert$
    \item The off-diagonal blocks (entries where $j \neq \ell$) encode \textit{cross-layer} interactions: the joint penalty of pruning expert $k$ in layer $j$ and expert $m$ in layer $\ell$ simultaneously
\end{itemize}

CHOPE strictly generalizes per-layer HOPE. Setting $\Phi_{L,j} = I$ for all $j$ and zeroing the off-diagonal blocks of $G$ (i.e. $G[(j,k),(\ell,m)] = 0$ for $j \neq \ell$) decouples layers entirely, recovering $L$ independent per-layer HOPE problems (Theorem~\ref{thm:qp}). Further zeroing the within-layer off-diagonal entries effectively recovers REAP.

Unlike per-layer HOPE (which enforces a fixed budget per layer), CHOPE optimizes a single global budget $B = \sum\limits_{\ell} \vert P^{\ell}\vert$ across all layers. The optimization allocates pruning non-uniformly, potentially pruning more aggressively in layers with weaker cooperative structure or lower propagation impact.

\subsection{Practical considerations}

\textbf{Jacobian approximation}: Computing the full $d \times d$ Jacobians $J_{A}^{\ell}$ and $J_{M}^{\ell}$ at every layer is prohibitively expensive for large models. A practical simplification is to set $\Phi_{L,j} = I$ for all $j$, corresponding to the assumption that errors propagate through the residual stream without transformation (i.e. Jacobians are approximately zero). For residual-stream architectures where the skip connection dominates, this is a reasonable approximation. Under this simplification, $G$ retains cross-layer interaction terms but does not differentially weight layers based on error propagation.

\textbf{Per-layer minimum constraints}: In practice, optimizing a global budget without constraints can lead to instability (e.g. over-pruning early layers). We propose adding per-layer minimum constraints to ensure at least $M$ experts (the routing width) survive in every layer:
\[
p^{*} = \argmin_{p}\; p^{\top} G\, p \quad \text{s.t.} \quad p \in \{0,1\}^{LE},\; \sum\limits_{k=1}^{LE} p_{k} = B,\; \sum\limits_{k=1}^{E} (1 - p_{k}^{j}) \geq M \;\;\forall\, j
\]

\textbf{Scalability}: The matrix $G$ has dimension $LE \times LE$. While larger than the per-layer $E \times E$ matrix, this remains tractable for modern QP solvers.

\subsection{Preliminary results}

Here, we present some preliminary results comparing the performance of CHOPE to layer-wise HOPE.

\begin{table}[h]
\centering
\begin{tabular}{lrcc}
\toprule
Method & Rate & Tulu Mean & SWE-Bench Pro \\
\midrule
Base & & 48.87 & 41.82 \\
\midrule
HOPE & 10\% & \textbf{48.49{\tiny$\pm$0.53}} & \textbf{43.18{\tiny$\pm$1.40}} \\
CHOPE & 10\% & 46.17{\tiny$\pm$0.72} & 40.87{\tiny$\pm$1.41} \\
\addlinespace
HOPE & 25\% & \textbf{48.51{\tiny$\pm$0.61}} & \textbf{44.58{\tiny$\pm$1.42}} \\
CHOPE & 25\% & 42.36{\tiny$\pm$1.08} & 27.29{\tiny$\pm$17.51} \\
\addlinespace
HOPE & 50\% & 42.93{\tiny$\pm$1.44} & \textbf{42.33{\tiny$\pm$1.23}} \\
CHOPE & 50\% & \textbf{48.12{\tiny$\pm$2.44}} & 22.95{\tiny$\pm$1.19} \\
\bottomrule
\end{tabular}
\caption{HOPE vs cross-layer HOPE (CHOPE) for Qwen3.5-35B-A3B (calib: Evol-CA). Scores $\times 100$ (mean and standard deviation across 3 independent calibration trials). The better method for each column is bolded.}
\label{tab:chope_qwen35_35b}
\end{table}

\begin{table}[h]
\centering
\begin{tabular}{lrcc}
\toprule
Method & Rate & Tulu Mean & SWE-Bench Pro \\
\midrule
Base & & 39.34 & 50.16 \\
\midrule
HOPE & 10\% & \textbf{40.14{\tiny$\pm$0.41}} & \textbf{52.44{\tiny$\pm$0.68}} \\
CHOPE & 10\% & 39.45{\tiny$\pm$0.76} & 46.04{\tiny$\pm$0.32} \\
\addlinespace
HOPE & 25\% & \textbf{39.59{\tiny$\pm$0.48}} & \textbf{52.29{\tiny$\pm$0.54}} \\
CHOPE & 25\% & 38.72{\tiny$\pm$1.06} & 47.10{\tiny$\pm$0.46} \\
\addlinespace
HOPE & 50\% & \textbf{43.75{\tiny$\pm$1.62}} & \textbf{43.84{\tiny$\pm$1.62}} \\
CHOPE & 50\% & 40.30{\tiny$\pm$0.88} & 42.40{\tiny$\pm$0.11} \\
\bottomrule
\end{tabular}
\caption{HOPE vs cross-layer HOPE (CHOPE) for Qwen3.5-122B-A10B (calib: Evol-CA). Scores $\times 100$ (mean and standard deviation across 3 independent calibration trials). The better method for each column is bolded.}
\label{tab:chope_qwen35_122b}
\end{table}

\begin{figure}[h]
\centering
\includegraphics[width=\columnwidth]{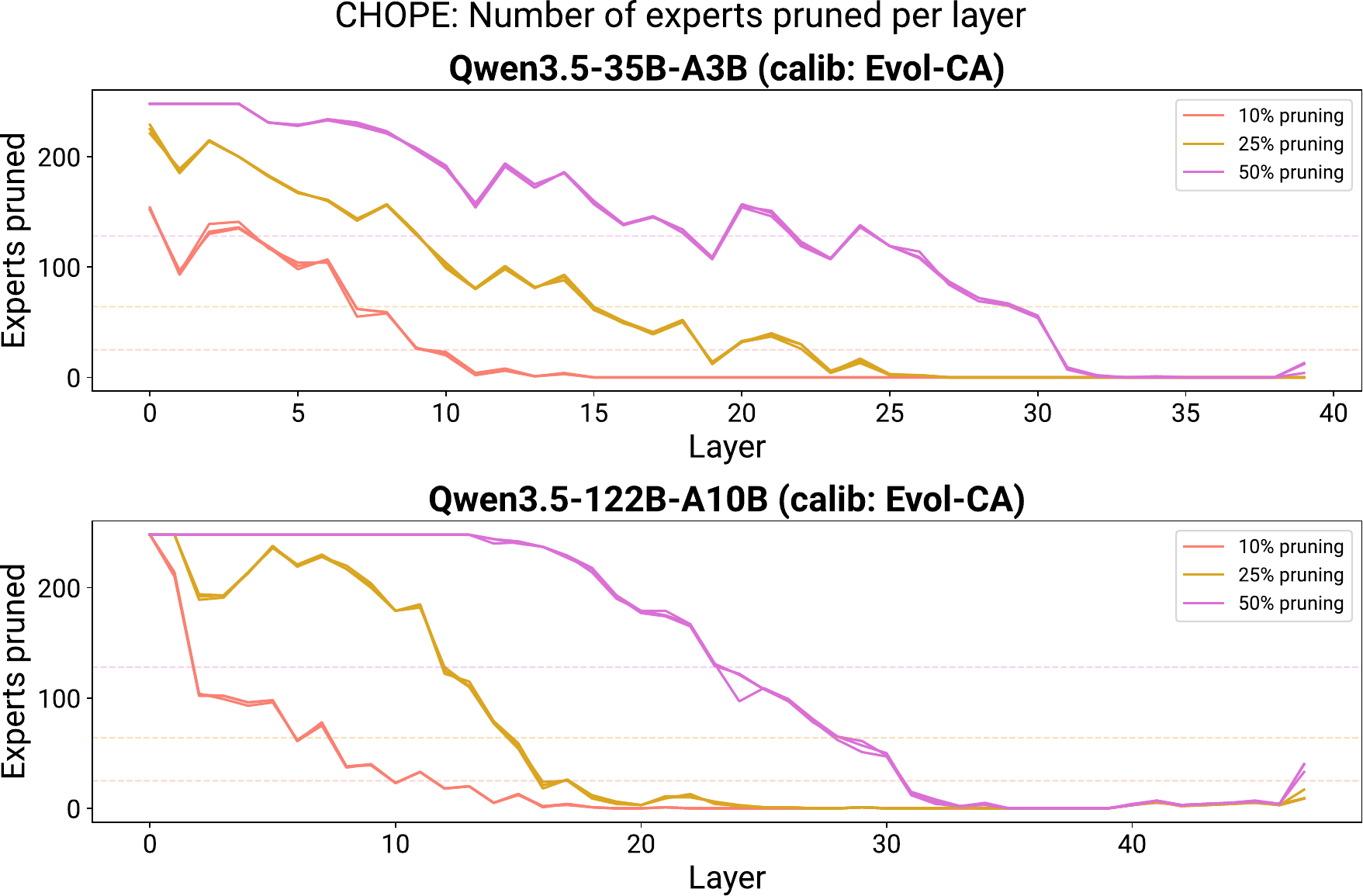}
\caption{\small Non-uniform expert pruning allocation by CHOPE (cross-layer HOPE). Each line indicates the number of experts pruned per layer (3 independent trials for each pruning rate); dashed horizontal lines indicate the uniform budget (used by per-layer HOPE).}
\label{fig:S7chope_pruned_per_layer}
\end{figure}

We implemented CHOPE and computed non-uniform prune-sets for Qwen3.5-35B-A3B and Qwen3.5-122B-A10B. For simplicity, we assume that Jacobians are $I$ (as described above), and do not limit the number of experts pruned per layer.

CHOPE's non-uniform allocation is consistent across trials, and in many conditions, CHOPE performed reasonably well, although not as well as per-layer HOPE (Tables~\ref{tab:chope_qwen35_35b}--\ref{tab:chope_qwen35_122b}). CHOPE-pruned models also occasionally experienced performance collapse.

CHOPE tended to prune a large number of experts in the earliest layers of the network (Figure~\ref{fig:S7chope_pruned_per_layer}), suggesting that the cross-layer objective's pruning allocation can introduce instability. Although promising, limiting the number of experts pruned (particularly in early layers) could help CHOPE perform better and avoid performance collapse. We leave further exploration of CHOPE for future work.

\clearpage

\section{Methods}
\label{app:methods}

\subsection{Calibration}

Expert usage statistics for HOPE were collected via a single forward pass over the calibration set. We hooked into each MoE layer's expert module and---for each token---recorded: 1) which experts were in the top-$K$ selected set, 2) their softmax-normalized gate weights, and 3) their output vectors. From these, we accumulated per-layer statistics in an $E \times E$ matrix: the total gate-weighted norm products $\sum\limits_{x:x\in\mathcal{X}_{i,j}}g_{i}(x) g_{j}(x) \Vert f_{i}(x)\Vert\Vert f_{j}(x)\Vert$ for all co-selected expert pairs $i,j$, along with co-selection counts $\vert\mathcal{X_{i,j}}\vert$.

We used two calibration sets:
\begin{itemize}[leftmargin=1em]
    \item \textbf{Evol-CodeAlpaca-v1} \citep{wizardcoder2023}: A coding-focused instruction dataset. We used all 111k prompts on the unpruned model.
    \item \textbf{SWE-bench Verified trajectories} \citep{swebench2024}: Agentic multi-turn coding trajectories collected from the unpruned model solving SWE-Bench Verified instances (500 instances). Full trajectories (including tool calls and observations) were reconstructed and tokenized.
\end{itemize}

For each calibration set (and each model), we ran 3 independent trials using the same prompts but different random seeds for model generation, yielding different routing patterns and thus different $F$ matrices per trial. For SWE-bench Verified, we truncated any traces longer than 100k tokens (rare).

\subsection{Quadratic-program solver}

Given the $E \times E$ interaction matrix $F$ for a layer, we solved the binary QP $\min_p p^\top F p$ subject to $p \in \{0,1\}^E$ and $\sum_k p_k = \vert P\vert$ via continuous relaxation. We relaxed $p \in \{0,1\}^E$ to $p \in [0,1]^E$ and solved with SciPy's SLSQP optimizer, using bounds $[0,1]$ per variable and a hard equality constraint on the sum. The initial point was set to $p_k = \frac{\vert P\vert}{E}$ (uniform). We used a tolerance of $10^{-12}$ and a maximum of 1000 iterations. The continuous solution was then rounded to binary by selecting the $\vert P\vert$ entries with the highest values.

For first-order baselines (REAP, MAN, EAN, Frequency), we computed their respective scalar scores using the same calibration data and collected the statistics as described in their respective publications, and sorted to select the $|P|$ experts with the lowest scores.

\subsection{Model pruning}

Pruned model checkpoints were created by physically removing expert parameters. For each pruned expert, we deleted its MLP weights (\texttt{gate\_up\_proj} and \texttt{down\_proj}) and sliced out the corresponding rows from the router's weight matrix. This reduced the model's parameter count and file size proportionally. The pruned model was saved as a standard HuggingFace checkpoint and loaded directly for inference.

\subsection{Tulu3 evaluation}

Pruned models were evaluated on the Tulu3-Dev benchmark suite using the OLMo Evaluation Suite (olmes) \citep{tulu3_2024}. We evaluated on 8 tasks: GSM8K, MATH, IFEval, MMLU, BBH, TruthfulQA, PopQA, and LiveCodeBench. We dropped Codex from the suite due to insufficient variation between models (LiveCodeBench is a more realistic set of coding tasks). Models were served via vLLM with a maximum sequence length of 8192 tokens and a generation budget of 2048 tokens per instance for computational efficiency. These are the same parameters we used for Evol-CodeAlpaca calibration on the unpruned models.

\subsection{SWE-Bench Pro evaluation}

Agentic coding evaluation was conducted on SWE-Bench Pro \citep{deng2025swebenchpro}, a benchmark of 731 enterprise-level software engineering instances. We used the Harbor evaluation framework with the OpenHands SDK agent (\texttt{openhands-sdk-bench}). The pruned model was served via vLLM and accessed by the agent as a hosted endpoint. We used the following evaluation parameters: \texttt{n\_attempts=1}, \texttt{max\_input\_tokens=253952}, \texttt{max\_output\_tokens=8192}, temperature 0.6, top-$p$ 0.95. The metric reported is the mean reward (resolve rate) across all 731 instances. These are the same parameters we used for SWE-Bench verified calibration on the unpruned models.

\subsection{Supervised fine-tuning}

Post-pruning SFT was conducted on Qwen3.5-122B-A10B pruned checkpoints (HOPE and REAP, at 50\% pruning rate, Evol-CodeAlpaca calibration). We fine-tuned on LiveCodeBench traces collected from the unpruned model, tokenized with sequence parallelism (8-way). Training used DeepSpeed ZeRO Stage 3 with full parameter fine-tuning (not LoRA), BF16 mixed precision, and Flash Attention 2. We used the following learning hyperparameters: learning rate $10^{-6}$, constant schedule with 20-step warmup, max gradient norm 1.0, global batch size $\sim$2M tokens/step (per-device batch size 1, gradient accumulation every 2 steps, sequence parallel size 8), trained for 500 steps ($\sim$1B tokens total). Pre-SFT and post-SFT evaluations used matched calibration trials for fair comparison.

\end{document}